\documentclass[twoside,11pt]{article}

\usepackage{amsthm}
\usepackage{myjmlr}

\usepackage{xspace}
\usepackage{amsmath}
\usepackage{euscript}
\usepackage{algorithm}
\usepackage{algorithmic}
\usepackage{graphicx}
\usepackage{xcolor}
\usepackage{subfigure}
\usepackage{url}
\usepackage{wrapfig}
\usepackage{mathrsfs}

\renewcommand{\vec}{\boldsymbol}
\renewcommand{\emph}{\textit}
\newcommand{\R}{{\mathscr{R}}}
\renewcommand{\H}{{H}}
\definecolor{darkblue}{rgb}{0.0,0.0,0.5}

\newcommand{\blue}[1]{\textcolor{black}{#1}}

\newcommand{\w}{\boldsymbol{w}}
\renewcommand{\t}{\boldsymbol{t}}
\renewcommand{\u}{\boldsymbol{u}}

\newcommand{\p}{r}
\renewcommand{\k}{k}
\renewcommand{\c}{c}
\renewcommand{\H}{H}

\newcommand{\vslack}{\boldsymbol{\xi}}
\newcommand{\slack}{\xi}
\newcommand{\x}{\boldsymbol{x}}
\newcommand{\y}{\boldsymbol{y}}
\newcommand{\E}{\mathbb{E}}
\newcommand{\loss}{V}

\newcommand{\vtheta}{{\boldsymbol{\theta}}}

\newcommand{\valpha}{{\boldsymbol{\alpha}}}
\newcommand{\margin}{{\gamma}}

\newcommand{\vxi}{{\boldsymbol{\xi}}}
\newcommand{\C}{C}

\newcommand{\one}{\mathbf{1}}
\newcommand{\zero}{\mathbf{0}}
\newcommand{\norm}[1]{\left\Vert#1\right\Vert}

\newtheorem*{corollary*}{Corollary}
\DeclareMathOperator*{\argmax}{argmax\xspace}
\DeclareMathOperator*{\argmin}{argmin\xspace}

\newcommand{\refprimal}{Optimization Problem (P)\xspace}
\newcommand{\refdual}{Optimization Problem (D)\xspace}

\newcommand{\refprimalcounter}{(P)\xspace}

\ShortHeadings{Non-sparse Regularization for Multiple Kernel Learning}{Kloft, Brefeld, Sonnenburg, and Zien}

\firstpageno{1}

\begin{document} 

\title{\blue{Non-Sparse Regularization for Multiple Kernel Learning}}

\author{\name Marius Kloft\thanks{Also at Machine Learning Group, Technische Universit\"at Berlin, Franklinstr. 28/29, FR 6-9, 10587 Berlin, Germany.} 
       \email mkloft@cs.berkeley.edu \\
       \addr University of California\\
       Computer Science Division\\
       Berkeley, CA 94720-1758, USA       
       \AND
       \name Ulf Brefeld \email brefeld@yahoo-inc.com \\
       \addr Yahoo! Research\\
       Avinguda Diagonal 177\\
       08018 Barcelona, Spain
       \AND
	   \name S\"oren Sonnenburg$^*$
	   \email soeren.sonnenburg@tuebingen.mpg.de\\
	   \addr Friedrich Miescher Laboratory\\
	   Max Planck Society \\
	   Spemannstr.\ 39, 
	   72076 T\"ubingen, Germany
       \AND
       \name Alexander Zien       \email zien@lifebiosystems.com \\
       \addr LIFE Biosystems GmbH \\
       Poststra{\ss}e 34 \\
       69115 Heidelberg, Germany 
}

\date{Received: date / Accepted: date}

\maketitle

\begin{abstract}%
Learning linear combinations of multiple kernels is an appealing
strategy when the right choice of features is unknown. Previous
approaches to multiple kernel learning (MKL) promote sparse kernel
combinations to support interpretability and scalability. 
Unfortunately, this $\ell_1$-norm
MKL is rarely observed to outperform trivial baselines in practical
applications. To allow for robust kernel mixtures, we generalize MKL
to arbitrary norms. We devise new insights on the connection
between several existing MKL formulations and develop two efficient
\emph{interleaved} optimization strategies for arbitrary norms, like
$\ell_p$-norms with $p>1$. Empirically, we demonstrate that the interleaved optimization
strategies are much faster compared to the commonly used wrapper approaches.
A theoretical analysis and an experiment on controlled artificial data experiment sheds light on
the appropriateness of sparse, non-sparse and $\ell_\infty$-norm MKL in various scenarios. 
Empirical applications of $\ell_p$-norm MKL to three	 real-world problems
from computational biology show that non-sparse MKL achieves
accuracies that go beyond the state-of-the-art.
\end{abstract} 


\section{Introduction}

Kernels allow to decouple machine learning from data representations.
Finding an appropriate data representation via a kernel function immediately opens the door 
to a vast world of powerful machine learning models \citep[e.g.][]{SchSmo02}
with many efficient and reliable off-the-shelf implementations.
This has propelled the dissemination of machine learning
techniques to a wide range of diverse application domains. 

Finding an appropriate data abstraction---or even engineering {\em the best} kernel---for the problem at hand is not always trivial, though. Starting with
cross-validation \citep{Stone1974}, which is probably the
most prominent approach to general model selection,
a great many approaches to selecting the right kernel(s) 
have been deployed in the literature.

Kernel target alignment \citep{Cristianini2002,twostage} aims
at learning the entries of a kernel matrix by using
the outer product of the label vector as the ground-truth.
\cite{ChaVapBouMuk02} and \cite{BousquetHerrmann2002}
minimize estimates of the generalization error of support vector
machines (SVMs) using a gradient descent algorithm over the set of 
parameters. 
\cite{OngEtAl2005} study hyperkernels on the space of kernels
and alternative approaches include selecting kernels by
DC programming \citep{ArgyriouEtAl2008} and semi-infinite 
programming \citep{OezoeguerAkyuezWeber2008,GehlerNowozin2008}.
Although finding non-linear kernel mixtures \citep{GoeAlp08,VarBab09} generally results in 
non-convex optimization problems, \cite{CorMohRos09b} show
that convex relaxations may be obtained for special cases.

However, learning arbitrary kernel combinations is 
a problem too general to allow for a general optimal solution---by focusing on a restricted scenario, it is possible
to achieve guaranteed optimality.
In their seminal work, \cite{LanCriGhaBarJor04} consider training an SVM
along with optimizing the linear combination of
several positive semi-definite matrices,
$K=\sum\nolimits_{m=1}^M \theta_m K_m,$
subject to the trace constraint $\text{tr}(K) \leq c$ and requiring
a valid combined kernel $K \succeq 0$.
This spawned the new field of \emph{multiple kernel learning} (MKL),
the automatic combination of several kernel functions.
\cite{LanCriGhaBarJor04} show that their specific version of the MKL task
can be reduced to a convex optimization problem,
namely a semi-definite programming (SDP) optimization problem.
Though convex, however, the SDP approach is computationally
too expensive for practical applications.
Thus much of the subsequent research focuses on 
devising more efficient optimization procedures.

One conceptual milestone for developing MKL into a tool of practical utility
is simply to constrain the mixing coefficients $\vtheta$
to be non-negative: by obviating the complex constraint $K \succeq 0$,
this small restriction allows one to transform the optimization problem
into a quadratically constrained program, hence drastically reducing the computational burden. 
While the original MKL objective is stated and optimized in 
dual space, alternative formulations have been studied. 
For instance, \cite{BacLanJor04} found a corresponding primal problem,
and \cite{Rub05} decomposed the MKL problem 
into a min-max problem that can be optimized 
by mirror-prox algorithms \citep{Nem04}.
The min-max formulation has been independently proposed by \cite{SonRaeSch05}. 
They use it to recast MKL training as a semi-infinite linear program.
Solving the latter with column generation \citep[e.g.,][]{NasSof96} amounts to
repeatedly training an SVM on a mixture kernel
while iteratively refining the mixture coefficients $\vtheta$.
This immediately lends itself to a convenient implementation by a wrapper approach.
These wrapper algorithms directly benefit from efficient SVM optimization routines 
\citep[cf., e.g.,][]{FanCheLin05,Joa99} and
are now commonly deployed in recent MKL solvers \citep[e.g.,][]{RakBacCanGra08,XuEtAl09}, 
thereby allowing for large-scale training \citep{SonRaeSch05,SonRaeSchSch06}.
However, the complete training of several SVMs
can still be prohibitive for large data sets.
For this reason, \cite{SonRaeSch05} also propose to interleave
the SILP with the SVM training which reduces the training time drastically.
 Alternative optimization schemes include
 level-set methods \citep{XuEtAl09} and  second order 
approaches \citep{ChaRak08}. 
\cite{SzaGraRak10}, \cite{Nathetal09}, and \cite{Bac09}
study composite and hierarchical kernel learning approaches.
Finally, \cite{ZieOng07} and 
\cite{JiSunJinYe09} provide extensions for 
multi-class and multi-label settings, respectively.


Today, there exist two major families of multiple kernel learning models.
The first is characterized by Ivanov regularization \citep{IvaVasTan02} over the mixing 
coefficients \citep{RakBacCanGra07,ZieOng07}.
\blue{For the Tikhonov-regularized optimization problem \citep{TikVasArs77},}
there is an additional parameter controlling
the regularization of the mixing coefficients \citep{VarRay07}.
 
All the above mentioned multiple kernel learning 
formulations promote \emph{sparse} solutions 
in terms of the mixing coefficients. The desire for sparse mixtures
originates in practical as well as theoretical reasons. First, sparse combinations
are easier to interpret. Second, irrelevant (and possibly expensive) kernels
functions do not need to be evaluated at testing time. 
Finally, sparseness appears
to be handy also from a technical point of view, as the additional simplex constraint
$\|\vtheta\|_1\leq 1$ simplifies derivations and turns the problem into a
linearly constrained program.
Nevertheless, sparseness is not always beneficial in practice and
sparse MKL is frequently observed to be outperformed by a regular SVM using
an unweighted-sum kernel $K=\sum_m K_m$ \blue{\citep{WSNips}.}

Consequently, despite all the substantial progress in the field of MKL,
there still remains an unsatisfied need for an approach that is
really useful for practical applications: a model that has a good chance
of improving the accuracy (over a plain sum kernel) together with
an implementation that matches today's standards (i.e., that can be
trained on 10,000s of data points in a reasonable time).
In addition, since the field has grown several competing MKL formulations,
it seems timely to consolidate the set of models.
In this article we argue that all of this is now achievable.

\subsection{\blue{Outline of the Presented Achievements}}


On the theoretical side, we cast multiple kernel learning as a general
regularized risk minimization problem for arbitrary convex loss 
functions, Hilbertian regularizers, and arbitrary norm-penalties on $\vtheta$. 
We first show that the above mentioned Tikhonov and 
Ivanov regularized MKL variants are equivalent in the sense
that they yield the same set of hypotheses.
Then we derive a dual representation and show that a variety of methods are special
cases of our objective.
Our optimization problem subsumes state-of-the-art approaches 
to multiple kernel learning, covering sparse and non-sparse MKL 
by arbitrary $p$-norm regularization ($1\leq p \leq \infty$) on 
the mixing coefficients as well as the incorporation of prior knowledge 
by allowing for non-isotropic regularizers.
As we demonstrate, the $p$-norm regularization includes both
important special cases (sparse $1$-norm and plain sum $\infty$-norm)
and offers the potential to elevate predictive accuracy over both of them.

With regard to the implementation, we introduce an appealing and efficient
optimization strategy which grounds on an exact
update in closed-form in the $\vtheta$-step; hence rendering expensive  
semi-infinite and first- or second-order gradient methods unnecessary.
By utilizing proven working set optimization for SVMs,
$p$-norm MKL can now be trained highly efficiently for all $p$;
in particular, we outpace other current $1$-norm MKL implementations.
Moreover our implementation employs kernel caching techniques,
which enables training on ten thousands of data points or thousands of
kernels respectively.
In contrast, most competing MKL software require all kernel matrices
to be stored completely in memory, which restricts these methods to small data sets
with  limited numbers of kernels.
Our implementation is freely available within the SHOGUN machine learning toolbox available at \url{http://www.shogun-toolbox.org/}.


Our claims are backed up by  experiments on artificial data and on a couple of real world data 
sets representing diverse, relevant and
challenging problems from the application domain bioinformatics.  Experiments on 
artificial data enable us to investigate the relationship between
properties of the true solution and the optimal choice of kernel
mixture regularization.  The real world problems include the
prediction of the subcellular localization of proteins, 
the (transcription) starts of genes, and the function of enzymes.
The results demonstrate (i) that combining kernels is now tractable on
large data sets, (ii) that it can provide cutting edge classification
accuracy, and (iii) that depending on the task at hand, 
different kernel mixture regularizations are
required for achieving optimal performance.
%

\blue{
In Appendix~\ref{sect:bounds} we present a first theoretical analysis of non-sparse MKL.
We introduce a novel $\ell_1$-to-$\ell_p$ conversion technique
and use it to derive generalization bounds.
Based on these, we perform a case study to compare a particular sparse with a non-sparse scenario.
}

\blue{A basic version of this work appeared in
NIPS 2009 \citep{KloBreSonZieLasMue09}. The present article additionally offers 
a more general and complete derivation of the
main optimization problem, exemplary applications thereof,
a simple algorithm based on a closed-form solution, 
technical details of the implementation,
a theoretical analysis, and additional experimental results. Parts of Appendix~\ref{sect:bounds} are based on \cite{KloRueBar10}  the present analysis however extends the previous publication by a novel conversion technique, an illustrative case study, and an improved presentation.}

\blue{Since its initial publication in \cite{KloBreLasSon08}, \cite{CorMohRos09a}, and \cite{KloBreSonZieLasMue09}, non-sparse MKL has been subsequently applied,  extended, and further analyzed by several researchers:
\cite{VarBab09} derive a projected gradient-based optimization method for $\ell_2$-norm MKL.
\cite{Yuetal}  present a more general dual view of $\ell_2$-norm MKL and show advantages of $\ell_2$-norm over an unweighted-sum kernel SVM on six bioinformatics data sets.
\cite{CorMohRos10} provide generalization bounds for $\ell_1$- and $\ell_{p\leq 2}$-norm MKL.
The analytical optimization method presented in this paper
was independently and in parallel discovered by \cite{Xuetal10} and has
also been studied in \cite{RothFischer} and \cite{YinCamDamGir09} for $\ell_1$-norm MKL,
and in \cite{SzaGraRak10} and \cite{Nathetal09} for composite kernel learning on small and medium scales.}

The remainder is structured as follows.  We derive
non-sparse MKL in Section \ref{SEC-derivation} and discuss
relations to existing approaches in Section \ref{SEC-applications}.
Section \ref{SEC-optimization} introduces the novel optimization 
strategy and its implementation. 
We report on our empirical results in
Section \ref{SEC-experiments}. 
Section \ref{SEC-conclusion} concludes.


\section{\blue{Multiple Kernel Learning -- A Regularization View}\label{SEC-derivation}}

In this section we cast multiple kernel learning into a unified
framework: we present a regularized loss minimization formulation with
additional norm constraints on the kernel mixing coefficients.  We
show that it comprises many popular MKL variants currently discussed in
the literature, including seemingly different ones.

We derive generalized dual optimization problems without making
specific assumptions on the norm regularizers or the loss function,
beside that the latter is convex. Our formulation covers binary
classification and regression tasks and can easily be extended to
multi-class classification and structural learning settings using
appropriate convex loss functions and joint kernel extensions.  Prior
knowledge on kernel mixtures and kernel asymmetries can be
incorporated by non-isotropic norm regularizers.



\subsection{Preliminaries}

We begin with reviewing the classical supervised learning setup. Given
a labeled sample $\mathcal{D} = \{(\x_i,y_i)\}_{i=1\ldots,n}$, where
the $\x_i$ lie in some input space $\mathcal{X}$ and $y_i\in\mathcal
Y\subset\mathbb R$, the goal is to find a hypothesis $h \in
\H$, 
that generalizes well on new and unseen data. Regularized risk
minimization returns a minimizer $h^*$,
\begin{equation}\label{rrm}
h^* \in \argmin\nolimits_h \,  \text{R}_{\text{emp}}(h) +\lambda\Omega(h)  ,\nonumber
\end{equation}
where $\text{R}_{\text{emp}}(h)=\frac{1}{n}\sum_{i=1}^n
\loss\left(h(\x_i),y_i\right)$ is the empirical risk of hypothesis $h$
w.r.t.~a convex loss function $\loss:\mathbb R\times\mathcal
Y\rightarrow\mathbb R$, $\Omega:\H \rightarrow\mathbb R$ is
a regularizer, and $\lambda> 0$ is a trade-off parameter.
We consider linear models of the form
\begin{align}
h_{\tilde{\w},b}(\x) = \langle\tilde\w,\psi(\x)\rangle+b \label{model},
\end{align}
together with a (possibly non-linear) mapping $\psi:\mathcal{X}\rightarrow\mathcal{H}$
to a Hilbert space $\mathcal{H}$ \citep[e.g.,][]{SchSmoMue98,MueMikRaeTsuSch01} and constrain the regularization to be of the form
$\Omega(h)=\frac{1}{2}\Vert\tilde\w\Vert_2^2$
which allows to kernelize the resulting models and algorithms.
We will later make use of kernel 
functions $\k(\x,\x')=\langle \psi(\x),\psi(\x') \rangle_{\mathcal{H}}$ to compute inner 
products in $\mathcal{H}$. 

\subsection{Regularized Risk Minimization with Multiple Kernels}

When learning with multiple kernels, we are given $M$ different
feature mappings $\psi_m:\mathcal{X}\rightarrow\mathcal{H}_m, ~ m=1,\ldots M$, each giving rise
to a reproducing kernel $\k_m$ of $\mathcal{H}_m$. Convex approaches to multiple kernel learning consider linear kernel mixtures $\k_\vtheta=\sum \theta_m \k_m$, $\theta_m\geq0$. 
Compared to Eq.~\eqref{model}, the primal model for learning with multiple kernels is extended to
\begin{align}\label{eq:mkl_model}
h_ {\tilde\w,b,\vtheta}(\x)\,=\,\sum_{m=1}^M \sqrt{\theta_m}\langle\tilde\w_m,\psi_m(\x)\rangle_{\mathcal{H}_m}+b\,=\,\langle\tilde\w,\psi_{\vtheta}(\x)\rangle_{\mathcal{H}}+b
\end{align}
where the parameter vector $\tilde\w$ and the composite feature map $\psi_{\vtheta}$ have
a block structure $\tilde\w=(\tilde\w_1^\top,\ldots,\tilde\w^\top_M)^\top$ and
$\psi_{\vtheta}=\sqrt{\theta_1}\psi_1\times\ldots\times\sqrt{\theta_M}\psi_M$, respectively. 

In learning with multiple kernels we aim at
minimizing the loss on the training data w.r.t.\ the optimal kernel mixture $\sum_{m=1}^M \theta_m\k_m$ in addition to regularizing $\vtheta$ to avoid overfitting.
Hence, in terms of regularized risk minimization, the optimization problem becomes
\begin{align} \inf_{\tilde\w, b, \vtheta: \vtheta \geq{\bf 0}} \quad  \frac{1}{n}\sum_{i=1}^n \loss\left(\sum_{m=1}^M \sqrt{\theta_m}\langle\tilde\w_m,\psi_m(\x_i)\rangle_{\mathcal{H}_m}+b,~y_i\right)+ \frac{\lambda}{2}\sum_{m=1}^M\Vert\tilde{w}_m\Vert_{\mathcal{H}_m}^2 + \tilde\mu\tilde\Omega[\vtheta],\label{333}
\end{align}
for $\tilde\mu> 0$. Note that the objective value of Eq. \eqref{333} is an upper
bound on the training error.
Previous approaches to multiple kernel learning employ regularizers of
the form $\tilde\Omega(\vtheta)=\Vert\vtheta\Vert_1$ to promote sparse
kernel mixtures. By contrast, we propose to use convex
regularizers of the form $\tilde\Omega(\vtheta)=\Vert\vtheta\Vert^2$, where 
$\Vert\cdot\Vert^2$ is an arbitrary norm in $\mathbb R^M$, possibly
allowing for non-sparse solutions and the incorporation of prior knowledge. The
non-convexity arising from the $\sqrt{\theta_m}\tilde\w_m$ product in the loss
term of Eq. \eqref{333}  is not inherent and can be resolved by substituting $\w_m\leftarrow
\sqrt{\theta_m}\tilde\w_m$. 
Furthermore, the regularization parameter and the sample
size can be decoupled by introducing $\tilde{C}=\frac{1}{n\lambda}$
(and adjusting $\mu\leftarrow \frac{\tilde\mu}{\lambda}$) which has
favorable scaling properties in practice. We obtain the following
convex optimization problem \citep{BoyVan04} 
that has also been considered by \citep{VarRay07} for hinge loss and an $\ell_1$-norm regularizer
\begin{equation}\label{eq:gen_varma}
  \inf_{\w, b,\vtheta:\vtheta\geq{\bf 0}} \quad  \tilde{C}\sum_{i=1}^n \loss\left(\sum_{m=1}^M \langle \w_m,\psi_m(\x_i)\rangle _{\mathcal{H}_m}+b,~y_i\right)+ \frac{1}{2}\sum_{m=1}^M\frac{\Vert\w_m\Vert_{\mathcal{H}_m}^2}{\theta_m} + \mu\Vert\vtheta\Vert^2,
\end{equation}
where we use the convention that $\frac{t}{0}=0$ if $t=0$ and $\infty$ otherwise. 

An alternative approach has been studied by
\cite{RakBacCanGra07} and \cite{ZieOng07}, again using hinge loss and
$\ell_1$-norm. They upper bound the value of the regularizer
$\Vert\vtheta\Vert_1\leq 1$ and incorporate the latter as an additional
constraint into the optimization problem. For $C> 0$, they arrive at the following problem which is the primary object of investigation in this paper.
\\
\\
\textbf{Primal MKL Optimization Problem}
\newcounter{storedequation} 
\let\storedtheequation=\theequation 
\setcounter{storedequation}{\value{equation}} 
\setcounter{equation}{0} 
\renewcommand{\theequation}{P}
\begin{align}
     \inf_{\w, b,\vtheta:\vtheta\geq {\bf 0}}  \quad  & C\sum_{i=1}^n \loss\Bigl(\sum_{m=1}^M \langle \w_m,\psi_m(\x_i)\rangle _{\mathcal{H}_m}+b,~y_i\Bigr)+ \frac{1}{2}\sum_{m=1}^M\frac{\Vert\w_m\Vert_{\mathcal{H}_m}^2}{\theta_m} \\
      \text{s.t.} \qquad  & \Vert\vtheta\Vert^2 \leq 1 .   \nonumber
\end{align}
\setcounter{equation}{\value{storedequation}}  
\renewcommand{\theequation}{\storedtheequation}

\blue{It is important to note here that, while the Ivanov regularization in \eqref{eq:gen_varma} has \emph{two} regularization parameters ($C$ and $\mu$), the above Tikhonov regularization \refprimalcounter  has only \emph{one}  ($C$ only).
Our first contribution shows that, despite the additional regularization parameter, both MKL variants are equivalent, in the sense that traversing the regularization paths yields the same binary classification functions.}

\begin{theorem}\label{th:reg-obj}
  Let $\Vert\cdot\Vert$ be a norm on $\mathbb R^M$, be $\loss$ a convex loss function. Suppose for the optimal $\w^*$ in \refprimal it holds $\w^*\neq \zero$. Then, for each pair $(\tilde{C},\mu)$ there exists $C> 0$ 
  such that for each optimal solution ($\w,b,\theta$) of Eq.~\eqref{eq:gen_varma} using $(\tilde C,\mu)$, we have that $(\w,b,\kappa\,\vtheta)$ is also an optimal solution of \refprimal using ${C}$, and vice versa, where $\kappa> 0$ is a multiplicative constant.
\end{theorem}

For the proof we need Prop.~\ref{prop:pareto}, which justifies switching from Ivanov to Tikhonov regularization, and back, 
if the regularizer is tight. We refer to Appendix~\ref{APP-proofs} for the proposition and its proof. \\

\begin{proof}\textbf{of Theorem~\ref{th:reg-obj}}
Let be $(\tilde{C},\mu)> 0$. In order to apply Prop.~\ref{prop:pareto} to \eqref{eq:gen_varma}, we show that 
condition \eqref{eq:constr_active} in Prop.~\ref{prop:pareto} is satisfied, i.e., that the regularizer is tight.

Suppose on the contrary, that \refprimal yields the same infimum regardless of whether we require
\begin{equation}\label{eq_aux}
  \Vert\vtheta\Vert^2 \leq 1 ,
\end{equation}
or not.
Then this implies that in the optimal point we have $\sum_{m=1}^M\frac{\Vert\w^*_m\Vert_2^2}{\theta^*_m} = 0$, hence,  
\begin{equation}\label{eq:contradiction}
\frac{\Vert\w^*_m\Vert_2^2}{\theta^*_m}=0, ~ ~ \forall ~ m=1,\ldots,M.
\end{equation}
Since all norms on $\mathbb R^M$ are equivalent \cite[e.g.,][]{Rud91}, there exists a $L<\infty$ such that $\Vert\vtheta^*\Vert_\infty \leq L \Vert\vtheta^*\Vert$. 
In particular, we have $\Vert\vtheta^*\Vert_\infty<\infty$, from which we conclude by \eqref{eq:contradiction}, that $\w_m=0$ holds for all $m$, which contradicts our assumption.

Hence, Prop.~\ref{prop:pareto} can be applied,\footnote{Note that after a coordinate transformation, we can assume that $\mathcal H$ is finite dimensional \cite[see][]{SchMikBurKniMueRaeSmo99}.} which yields that \eqref{eq:gen_varma} is equivalent to 
\begin{align*}
  \inf_{\w, b,\vtheta}   \quad & \tilde{C}\sum_{i=1}^n V\Bigl(\sum_{m=1}^M \langle\w_m,\psi_m(\x)\rangle+b,~y_i\Bigr)+ \frac{1}{2}\sum_{m=1}^M\frac{\Vert\w_m\Vert_2^2}{\theta_m} \\
  \text{s.t.} \quad &  \Vert\vtheta\Vert^2 \leq \tau \nonumber , 
\end{align*}
for some $\tau> 0$.  Consider the optimal solution
$(\w^\star,b^\star,\vtheta^\star)$ corresponding to a given
parametrization $(\tilde{C},\tau)$.  For any $\lambda> 0$, the
bijective transformation $(\tilde{C},\tau) \mapsto
(\lambda^{-1/2}\tilde{C},\lambda\tau)$ will yield
$(\w^\star,b^\star,\lambda^{1/2}\vtheta^\star)$ as optimal solution.
 Applying the transformation 
with $\lambda := 1/\tau$ and
setting $C=\tilde{C}\tau^{\frac{1}{2}}$ as well as $\kappa=\tau^{-1/2}$ yields
\refprimal, which was to be shown.
\end{proof}

\cite{ZieOng07} also show that the MKL optimization problems by 
\cite{BacLanJor04}, \cite{SonRaeSchSch06}, and
their own formulation are equivalent. As a main implication of Theorem
\ref{th:reg-obj} and by using the result of Zien and Ong it follows that the optimization
problem of Varma and Ray \citep{VarRay07} lies in the same equivalence class as
\citep{BacLanJor04,SonRaeSchSch06,RakBacCanGra07,ZieOng07}.
In addition, our result shows the coupling between trade-off parameter $C$ and the
regularization parameter $\mu$ in Eq.~\eqref{eq:gen_varma}: tweaking
one also changes the other and vice versa.
Theorem \ref{th:reg-obj} implies that optimizing $C$ in
\refprimal implicitly searches the
regularization path for the parameter $\mu$ of
Eq.~\eqref{eq:gen_varma}.  In the remainder, we will therefore focus
on the formulation in \refprimal, as a single
parameter is preferable in terms of model selection.  

\subsection{MKL in Dual Space}\label{sec:dual}

In this section we study the generalized MKL approach of the previous section in the dual space.
Let us begin with rewriting \refprimal by expanding the decision values into slack variables as follows
\begin{align}\label{RakoZienConvex}
     \inf_{\w, b,\t,\vtheta:\vtheta\geq{\bf 0}}  & \quad  \C\sum_{i=1}^n \loss\left(t_i,~y_i\right)+ \frac{1}{2}\sum_{m=1}^M\frac{\Vert\w_m\Vert_{\mathcal{H}_m}^2}{\theta_m} \\
     \text{s.t.} \quad ~ & \quad \forall i: ~  \sum_{m=1}^M \langle \w_m,\psi_m(\x_i)\rangle _{\mathcal{H}_m}+b =t_i ~ ; \quad \Vert\vtheta\Vert^2 \leq 1, \nonumber
\end{align}
where $\Vert\cdot\Vert$ is an arbitrary norm in $\mathbb R^m$ and $\Vert\cdot\Vert_{\mathcal{H}_M}$ denotes the Hilbertian norm of $\mathcal{H}_m$.
Applying Lagrange's theorem re-incorporates the constraints into the
objective by introducing Lagrangian multipliers
$\valpha\in\mathbb R^n$ and $\beta\in\mathbb R_+$. \footnote{Note that, in contrast to the standard SVM dual deriviations, here $\valpha$ is a variable that ranges over all of $\mathbb R^n$, as it is incorporates an equality constraint.}
The Lagrangian saddle point problem is then given by
\begin{align}\label{lagr}
\sup_{\valpha,\beta:\beta\geq 0} ~ \inf_{\w,b,\t,\vtheta\geq 0} \quad & \C\sum_{i=1}^n \loss\left(t_i,~y_i\right)+ \frac{1}{2}\sum_{m=1}^M\frac{\Vert\w_m\Vert_{\mathcal{H}_m}^2}{\theta_m} \\
& -\sum_{i=1}^n\alpha_i\left(\sum_{m=1}^M \langle \w_m,\psi_m(\x_i)\rangle _{\mathcal{H}_m}+b -t_i\right) + \beta\left(\frac{1}{2}\Vert\vtheta\Vert^2 -\frac{1}{2}\right) \nonumber .
\end{align}
Denoting the Lagrangian by $\mathcal L$ and setting its first partial derivatives with respect to 
$\w$ and $b$ to $0$ reveals the optimality conditions 
\begin{align}
  \refstepcounter{equation}
   &\one^\top\valpha=0 \label{eq:opt_v.a} \tag{\theequation a}; \\
   &\vec w_m = \theta_m\sum_{i=1}^n \alpha_i \psi_m(\vec x_i), ~~ \forall\hspace{1mm} m=1,\ldots, M.   \label{eq:opt_v.b} \tag{\theequation b} 
\end{align}
Resubstituting the above equations yields
\begin{eqnarray*}
  \sup_{\valpha:~\one^\top\valpha=0, ~ \beta:\beta\geq 0} ~ \inf_{\t,\vtheta\geq 0}\quad \C\sum_{i=1}^n \left(\loss\left(t_i,~y_i\right)+\alpha_i t_i\right) 
  -\frac{1}{2}\sum_{m=1}^M \theta_m\valpha^\top K_m \valpha + \beta\left(\frac{1}{2}\Vert\vtheta\Vert^2 -\frac{1}{2}\right) \nonumber ,
\end{eqnarray*}
which can also be written in terms of  unconstrained $\vtheta$, because
the supremum with respect to $\vtheta$ is  attained for non-negative $\vtheta\geq 0$. We arrive at
\begin{eqnarray*}
    \sup_{\valpha:~\one^\top\valpha=0,~\beta\geq 0}  -\C\sum_{i=1}^n \sup_{t_i}\left(-\frac{\alpha_i}{\C}t_i-\loss\left(t_i,~y_i\right)\right)
  -\beta\sup_{\vtheta}\left(\frac{1}{2\beta}\sum_{m=1}^M \theta_m\valpha^\top K_m \valpha -\frac{1}{2}\Vert\vtheta\Vert^2\right) -\frac{1}{2}\beta \nonumber .
\end{eqnarray*}
As a consequence, we now may express the Lagrangian as\footnote{We employ the notation $s=(s_1,\ldots,s_M)^\top=(s_m)_{m=1}^M$ for $s\in \mathbb R^M$.}
\begin{eqnarray}\label{lagr_beta}
  \sup_{\valpha:~\one^\top\valpha=0,~\beta\geq 0} ~ -\C\sum_{i=1}^n \loss^*\left(-\frac{\alpha_i}{\C},~y_i\right)
  -\frac{1}{2\beta}\left\Vert\frac{1}{2} \left(\valpha^\top K_m \valpha\right)_{m=1}^M \right\Vert^2_* -\frac{1}{2}\beta ,
\end{eqnarray}
where $h^*(\x)=\sup_{\u} \x^\top\u -h(\u)$ denotes the Fenchel-Legendre
conjugate of a function $h$ and $\Vert\cdot\Vert_*$ denotes the \emph{dual norm},
i.e., the norm defined via the identity $\frac{1}{2}\Vert\cdot\Vert^2_*:=\left(\frac{1}{2}\Vert\cdot\Vert^2\right)^* $. 
In the following, we call $\loss^*$ the \emph{dual loss}.
Eq.~\eqref{lagr_beta} now has to be maximized with respect to the dual variables $\valpha,\beta$, subject to $\one^\top\valpha=0$ and $\beta\geq0$.
Let us ignore for a moment the non-negativity constraint on $\beta$ and solve $\partial\mathcal{L}/\partial{\beta}=0$ for the unbounded 
$\beta$. Setting the partial derivative to zero allows to express the optimal $\beta$ as
\begin{align}\label{eq:beta_opt}
\beta=\left\Vert\frac{1}{2} \left(\valpha^\top K_m\valpha\right)_{m=1}^M\right\Vert_*.
\end{align}
Obviously, at optimality, we always have $\beta\geq 0$. We thus discard the corresponding constraint from the optimization problem and plugging Eq. \eqref{eq:beta_opt} into  Eq.~\eqref{lagr_beta} results in the following \emph{dual} optimization problem which now solely depends on $\valpha$:
\\
\\
\textbf{Dual MKL Optimization Problem} \quad
\newcounter{dualequation} 
\let\dualtheequation=\theequation 
\setcounter{dualequation}{\value{equation}} 
\setcounter{equation}{0} 
\renewcommand{\theequation}{D}
\begin{align} 
  \hspace{-0.5cm}
  \sup_{\valpha:~\one^\top\valpha=0} ~ -\C\sum_{i=1}^n \loss^*\left(-\frac{\alpha_i}{\C},~y_i\right)
  -\frac{1}{2}\left\Vert \left(\valpha^\top K_m \valpha\right)_{m=1}^M \right\Vert_* .
\end{align}
\setcounter{equation}{\value{dualequation}}  
\renewcommand{\theequation}{\dualtheequation}

The above dual generalizes multiple kernel learning to arbitrary convex loss functions and norms.\footnote{We can even employ non-convex losses and still the dual will be a convex problem; however, it might suffer from a duality gap.}
Note that if the loss function is continuous (e.g., hinge loss), the supremum is also a maximum. The threshold $b$ can be recovered from the solution by applying the KKT conditions.

The above dual can be characterized as follows. We start by
noting that the expression in \refdual is a composition of
two terms, first, the left hand side term, which depends on the
conjugate loss function $\loss^*$, and, second, the right hand side
term which depends on the conjugate norm.  The right hand side can be
interpreted as a regularizer on the quadratic terms that, according to the chosen norm, smoothens the solutions.
Hence we have a decomposition of the dual into a loss term
(in terms of the dual loss) and a regularizer (in terms of the dual
norm). For a specific choice of a pair $(\loss,\Vert\cdot\Vert)$ we can
immediately recover the corresponding dual by computing the pair of
conjugates $(\loss^*,\Vert\cdot\Vert_*)$ \cite[for a comprehensive list of dual losses see][Table~3]{RifLip07}.
In the next section, this is
illustrated by means of well-known loss functions and regularizers.

\blue{
At this point we would like to highlight some properties of \refdual that arise due to our dualization technique.  
While approaches that firstly apply the representer theorem and secondly optimize in the primal such as \cite{Chapelle2006} also can employ general loss functions, the resulting loss terms depend on all optimization variables.
By contrast, in our formulation the dual loss terms are of a much simpler structure and they only depend on a single optimization variable $\alpha_i$. A similar dualization technique yielding singly-valued dual loss terms is presented in \cite{RifLip07}; it is based on 
Fenchel duality and limited to strictly positive definite kernel matrices. Our technique, which uses Lagrangian duality, extends the latter by allowing for positive semi-definite kernel matrices.
}


\section{Instantiations of the Model}\label{SEC-applications}

In this section we show that existing MKL-based learners are subsumed by the generalized formulation in \refdual.

\subsection{Support Vector Machines with Unweighted-Sum Kernels}\label{example1}
First we note that the support vector machine with an unweighted-sum
kernel can be recovered as a special case of our model. To see this, we consider
the regularized risk minimization problem using the hinge loss function
$\loss(t,y)=\max(0,1-ty)$ and the regularizer $\Vert\theta\Vert_\infty$.
We then can obtain the corresponding dual in terms of
Fenchel-Legendre conjugate functions as follows.

We first note that the dual loss of the hinge loss is $\loss^*(t,y)=\frac{t}{y}$ if
$-1\leq\frac{t}{y}\leq 0$ and $\infty$ elsewise \cite[][Table~3]{RifLip07}.
Hence, for each $i$ the term $\loss^*\left(-\frac{\alpha_i}{\C},~y_i\right)$ of
the generalized dual, i.e., \refdual, translates to 
$-\frac{\alpha_i}{Cy_i}$, provided that $0\leq \frac{\alpha_i}{y_i}\leq C$. Employing a variable substitution of the form
$\alpha^{\text{new}}_i=\frac{\alpha_i}{y_i}$, \refdual translates to
\begin{align}\label{dual_hinge}
   \max_\valpha ~ ~\one^\top\valpha
  -\frac{1}{2}\left\Vert\left(\valpha^\top Y K_mY\valpha \right)_{m=1}^M\right\Vert_*, \quad \text{s.t.} \quad {\y^\top\valpha=0 ~ ~ \text{and} ~ ~ \zero\leq\valpha\leq C\one}  ,
\end{align}
where we denote $Y={\rm diag}(y)$. The primal $\ell_\infty$-norm penalty $\Vert\vtheta\Vert_\infty$ is dual to $\Vert\vtheta\Vert_1$, 
hence, via the identity $\Vert\cdot\Vert_*=\Vert\cdot\Vert_1$ the right hand side of the last equation translates to $\sum_{m=1}^M\valpha^\top Y K_m Y\valpha$. 
Combined with \eqref{dual_hinge} this leads to the dual
\begin{align*}
   \sup_\valpha ~ ~\one^\top\valpha
  -\frac{1}{2} \sum_{m=1}^M\valpha^\top Y K_m Y\valpha, \quad \text{s.t.} \quad {\y^\top\valpha=0 ~ ~ \text{and} ~ ~ \zero\leq\valpha\leq C\one}  ,
\end{align*}
which is precisely an SVM with an unweighted-sum kernel.

\subsection{QCQP MKL of Lanckriet et al. (2004)}
A common approach in multiple kernel learning is to employ  regularizers of the form 
\begin{equation}\label{eq:ell1}
  \Omega(\vtheta)=\Vert\vtheta\Vert_1.
\end{equation}
This so-called $\ell_1$-norm regularizers are  specific instances of \emph{sparsity-inducing} regularizers. 
The obtained kernel mixtures  usually have a considerably large fraction of zero entries,
and hence equip the MKL problem by the favor of interpretable solutions.
Sparse MKL is a special case of our framework; to see this, note that the conjugate of \eqref{eq:ell1} is $\Vert\cdot\Vert_\infty$. 
Recalling the definition of an $\ell_p$-norm, the right hand side of \refdual translates
to $\max_{m\in\{1,\ldots,M\}}\valpha^\top Y K_mY\valpha$.
The maximum can subsequently be expanded into a slack variable $\slack$, resulting in 
\begin{align*}  
   \sup_{\valpha,\vslack} &\quad \one^\top\valpha  -\slack \\
  \text{s.t.} & \quad \forall\hspace{1mm} m: \;\; \frac{1}{2}\valpha^\top Y K_m Y \valpha\leq\xi ~; \quad \y^\top\valpha=0 ~ ; \quad \zero\leq\valpha\leq C\one ,\nonumber
\end{align*}
which is the original QCQP formulation of MKL, firstly given by \cite{LanCriGhaBarJor04}.

\subsection{$\ell_p$-Norm MKL}
Our MKL formulation also allows for robust kernel mixtures by
employing an $\ell_{p}$-norm constraint with $p>1$, rather than an $\ell_1$-norm
constraint, on the mixing coefficients \citep{KloBreSonZieLasMue09}.
The following identity holds
\begin{align*}
\left( \frac{1}{2}\Vert\cdot\Vert_p^2 \right)^*=\frac{1}{2}\Vert\cdot\Vert^2_{p^*}, 
\end{align*}
where $p^*:=\frac{p}{p-1}$ \vspace{0.5mm} is the conjugated exponent of $p$,
and we obtain for the dual norm of the $\ell_p$-norm:
$\Vert\cdot\Vert_*=\Vert\cdot\Vert_{p^*}$.
This leads to the dual problem
\begin{align*}
  \sup_{\valpha:\one^\top\valpha=\zero} ~ -\C\sum_{i=1}^n \loss^*\left(-\frac{\alpha_i}{\C},~y_i\right)
  -\frac{1}{2}\left\Vert \left(\valpha^\top K_m \valpha\right)_{m=1}^M \right\Vert_{p^*}.
\end{align*}
In the special case of hinge loss minimization, we obtain the optimization problem
\begin{align*}
   \sup_\valpha ~ ~\one^\top\valpha
  -\frac{1}{2}\left\Vert \left(\valpha^\top Y K_m Y\valpha\right)_{m=1}^M \right\Vert_{p^*},	 \quad \text{s.t.} \quad {\y^\top\valpha=0 ~ ~ \text{and} ~ ~ \zero\leq\valpha\leq C\one}  .
\end{align*}


\subsection{A Smooth Variant of Group Lasso}
\cite{YuaLin06} studied the following optimization problem for the special case $\mathcal H_m=\mathbb R^{d_m}$ and $\psi_m ={\rm id}_{\mathbb R^{d_m}}$, also known as group lasso,
\begin{align}\label{eq:groupLasso}
     \min_{\w}  \quad  \frac{C}{2}\sum_{i=1}^n \left(y_i-\sum_{m=1}^M \langle \w_m,\psi_m(\x_i)\rangle _{\mathcal{H}_m}\right)^2 + \frac{1}{2}\sum_{m=1}^M\Vert\w_m\Vert_{\mathcal{H}_m}.
\end{align}
The above problem has been solved by active set methods in the primal \citep{RotFis08}. We sketch an alternative approach based on dual optimization.
First, we note that Eq. \eqref{eq:groupLasso} can be equivalently expressed as \cite[][Lemma 26]{MicPon05}
\begin{align*}
     \inf_{\w, \vtheta:\vtheta\geq {\bf 0}}  \quad   \frac{C}{2}\sum_{i=1}^n \left(y_i-\sum_{m=1}^M \langle \w_m,\psi_m(\x_i)\rangle _{\mathcal{H}_m}\right)^2 + \frac{1}{2}\sum_{m=1}^M\frac{\Vert\w_m\Vert_{\mathcal{H}_m}^2}{\theta_m}, \quad \text{s.t.} \quad   \Vert\vtheta\Vert^2_1 \leq 1.
\end{align*}
The dual of $\loss(t,y)=\frac{1}{2}(y-t)^2$ is $\loss^*(t,y)=\frac{1}{2}t^2+ty$ and thus the corresponding group lasso dual can be written as
\begin{align}\label{groupLasso}
   \max_\valpha \quad  \y^\top\valpha -\frac{1}{2C}\Vert\valpha\Vert_2^2 
  -\frac{1}{2}\left\Vert \left(\valpha^\top Y K_m Y\valpha\right)_{m=1}^M \right\Vert_\infty  ,
\end{align}
which can be expanded into the following QCQP
\begin{align}  
   \sup_{\valpha,\xi} &\quad \y^\top\valpha -\frac{1}{2C}\Vert\valpha\Vert_2^2 -\blue{\xi} \\
  \text{s.t.} & \quad \forall \hspace{0.3mm}m: \quad \frac{1}{2}\valpha^\top Y K_m Y \valpha\leq\blue{\xi} \nonumber .
\end{align}
For small $n$, the latter formulation can be handled efficiently by QCQP solvers.
However, the quadratic constraints caused by the non-smooth $\ell_\infty$-norm in the objective still are computationally too demanding. 
As a remedy, we propose the following unconstrained variant based on $\ell_p$-norms ($ 1<p<\infty$), given by
\begin{align*}
   \max_\valpha \quad  \y^\top\valpha -\frac{1}{2C}\Vert\valpha\Vert_2^2 
  -\frac{1}{2}\left\Vert \left(\valpha^\top Y K_m Y\valpha\right)_{m=1}^M \right\Vert_{p^*}  .
\end{align*}
\blue{
It is straight forward to verify that the above objective function is  differentiable in any $\valpha\in\mathbb R^n$ (in particular, notice that the $\ell_p$-norm  function  is  differentiable for $1<p<\infty$) 
and hence the above optimization problem can be solved very efficiently by, for example, limited memory quasi-Newton descent methods \citep{LiuNoc1989}.
}

\subsection{Density Level-Set Estimation}
Density level-set estimators are frequently used for anomaly/novelty detection tasks  \citep{MarSin03a,MarSin03b}. 
Kernel approaches, such as one-class SVMs \citep{SchPlaShaSmoWil01} and Support Vector Domain Descriptions \citep{TaxDui99a} 
can be cast into our MKL framework by employing  loss functions of the form $\loss(t)=\max(0,1-t)$. This gives rise to the primal 
\begin{align*}
     \inf_{\w, \vtheta:\vtheta\geq {\bf 0}}  \quad  C\sum_{i=1}^n \max\left(0,\sum_{m=1}^M \langle \w_m,\psi_m(\x_i)\rangle _{\mathcal{H}_m}\right)+ \frac{1}{2}\sum_{m=1}^M\frac{\Vert\w_m\Vert_{\mathcal{H}_m}^2}{\theta_m}, \quad \text{s.t.} \quad   \Vert\vtheta\Vert^2 \leq 1.
\end{align*}
Noting that the dual loss is $\loss^*(t)=t$ if $-1\leq t\leq 0$ and $\infty$ elsewise, we obtain the following generalized dual
\begin {align*}
   \sup_{\valpha} ~ ~\one^\top\valpha
  -\frac{1}{2}\left\Vert \left(\valpha^\top K_m\valpha\right)_{m=1}^M \right\Vert_{p^*}, \quad \text{s.t.} \quad {\zero\leq\valpha\leq C\one}  ,
\end{align*}
which has been studied by \cite{SonRaeSchSch06} and \cite{RakBacCanGra08} for $\ell_1$-norm, and by \cite{KloNakBre09} for $\ell_p$-norms.

\subsection{Non-Isotropic Norms}
\label{matrixnorm}
In practice, it is often desirable for an expert to incorporate prior knowledge about the problem domain.
\blue{For instance, an expert could provide estimates of the interactions of kernels $\{K_1,...,K_M\}$
in the form of an $M\times M$ matrix $E$. Alternatively, $E$ could be obtained by computing  pairwise kernel alignments
$E_{ij} = \frac{<K_i,K_j>}{\Vert K_i\Vert ~ \Vert K_j \Vert}$ given a dot product on the space of kernels such
as the Frobenius dot product \citep{OngEtAl2005}.
In a third scenario,  $E$ could be a diagonal matrix encoding the a priori importance of kernels---it might be known 
from pilot studies that a subset of the employed kernels is inferior to the remaining ones.}

All those scenarios can be easily handled within our framework by considering non-isotropic regularizers of the form\footnote{This idea is inspired by the Mahalanobis distance \citep{Mah36}.}
$$\Vert\vtheta\Vert_{E^{-1}}=\sqrt{\vtheta^\top E^{-1}\vtheta} ~  ~ \text{with}  ~ ~ E \succ 0 ,$$
where $E^{-1}$ is the matrix inverse of $E$.
The dual norm is again defined via 
$\frac{1}{2}\Vert\cdot\Vert^2_*:=\left(\frac{1}{2}\Vert\cdot\Vert^2_{E^{-1}}\right)^*$ and the following easily-to-verify identity,
\begin{align*}
  \left(\frac{1}{2}\Vert\cdot\Vert^2_{E^{-1}}\right)^*=\frac{1}{2}\Vert\cdot\Vert_{E}^2,
\end{align*}
leads to the dual,
\begin{align*}
  \sup_{\valpha:\one^\top\valpha=\zero} ~ -\C\sum_{i=1}^n \loss^*\left(-\frac{\alpha_i}{\C},~y_i\right)
  -\frac{1}{2}\left\Vert \left(\valpha^\top K_m \valpha\right)_{m=1}^M \right\Vert_{E},
\end{align*}
which is the desired non-isotropic MKL problem.


\section{Optimization Strategies}
\label{SEC-optimization}


The dual as given in \refdual does not lend itself to efficient
large-scale optimization in a straight-forward fashion, for instance
by direct application of standard approaches like gradient descent.
Instead, it is beneficial to exploit the structure of the MKL cost
function by alternating between optimizing w.r.t.~the mixings
$\vtheta$ and w.r.t.~the remaining variables.  Most recent MKL
solvers \cite[e.g.,][]{RakBacCanGra08,XuEtAl09,Nathetal09} do so by
setting up a two-layer optimization procedure: a master problem, which
is parameterized only by $\vtheta$, is
solved to determine the kernel mixture; to solve this master problem,
repeatedly a slave problem is solved which amounts to training a
standard SVM on a mixture kernel. Importantly, for the slave problem,
the mixture coefficients are fixed, such that conventional, efficient
SVM optimizers can be recycled. Consequently these two-layer procedures
are commonly implemented as \emph{wrapper} approaches.  Albeit
appearing advantageous, wrapper methods suffer from two
shortcomings: (i) Due to kernel cache limitations, the kernel matrices
have to be pre-computed and stored or many kernel computations have
to be carried out repeatedly, inducing heavy wastage of either memory
or time. (ii) The slave problem is always optimized to the end (and
many convergence proofs seem to require this), although most of the
computational time is spend on the non-optimal mixtures. 
Certainly suboptimal slave solutions would already
suffice to improve far-from-optimal $\vtheta$ in the master problem.

Due to these problems, MKL is prohibitive when learning with a
multitude of kernels and on large-scale data sets as commonly
encountered in many data-intense real world applications such as
bioinformatics, web mining, databases, and computer security.
The optimization approach presented in this paper
decomposes the MKL problem into smaller
subproblems \citep{Pla99,Joa99,FanCheLin05} by establishing a
wrapper-like scheme \emph{within} the decomposition algorithm.


Our algorithm is embedded into the large-scale framework of
\cite{SonRaeSchSch06} and extends it to the optimization of non-sparse
kernel mixtures induced by an $\ell_p$-norm penalty.
Our  strategy alternates between minimizing the primal problem
\eqref{RakoZienConvex} w.r.t.~$\vtheta$ via a simple analytical update formula and with incomplete optimization
w.r.t.~all other variables which, however, is performed in terms of the
dual variables $\valpha$. 
Optimization w.r.t.~$\valpha$ is performed by
chunking optimizations with minor iterations.  
Convergence of our algorithm is proven under typical technical regularity assumptions.
\subsection{A Simple Wrapper Approach Based on an Analytical Update}\label{opt_alex}

We first present an easy-to-implement wrapper version of our optimization approach to 
multiple kernel learning. The interleaved decomposition algorithm is deferred to the next section.
To derive the new algorithm, we first revisit the primal problem, i.e.
\begin{equation*}
   \inf_{\w, b,\vtheta:\vtheta\geq {\bf 0}}  \quad  C\sum_{i=1}^n \loss\left(\sum_{m=1}^M \langle \w_m,\psi_m(\x_i)\rangle _{\mathcal{H}_m}+b,~y_i\right)+ \frac{1}{2}\sum_{m=1}^M\frac{\Vert\w_m\Vert_{\mathcal{H}_m}^2}{\theta_m}, \quad \text{s.t.} \quad   \Vert\vtheta\Vert^2 \leq 1. ~ ~ ~ \refprimalcounter
\end{equation*}
In order to obtain an efficient optimization strategy, we divide the variables in the above OP into two groups, $(\w,b)$ on one hand and $\vtheta$ on the other.
In the following we will derive an algorithm which alternatingly operates on those two groups via a block coordinate descent algorithm, also known as the \emph{non-linear block Gauss-Seidel method}. Thereby the optimization w.r.t.~$\vtheta$ will be carried out analytically and the $(\w,b)$-step will be computed in the dual, if needed. 

The basic idea of our first approach is that for a given, fixed set of primal variables $(\w,b)$, the
optimal $\vtheta$ in the primal problem~\refprimalcounter
can be  calculated analytically. In the subsequent derivations we employ non-sparse norms of the form 
$\Vert\vtheta\Vert_p= (\sum_{m=1}^M \theta_m^p)^{{1}/{p}}$, $1<p<\infty$. \footnote{While the reasoning also holds for  weighted $\ell_p$-norms, the extension to
more general norms, such as the ones described in Section~\ref{matrixnorm}, is left for future work.}

The following proposition gives an analytic update formula for $\vtheta$ given fixed remaining variables $(\w,b)$ and will
become the core of our proposed  algorithm.

\begin{proposition}\label{prop:directBetasOpt}
  Let  $\loss$ be a convex loss function, be $p>1$. Given fixed (possibly suboptimal) $\w\neq \zero$ and $b$, the minimal $\vtheta$ in \refprimal is attained for
  \begin{equation}\label{directBetasOpt}
    \theta_m = \frac{\Vert\w_m\Vert_{\mathcal{H}_m}^{\frac{2}{p+1}}}
    {\left(\sum_{m'=1}^M\Vert\w_{m'}\Vert_{\mathcal{H}_{m'}}^{\frac{2p}{p+1}}\right)^{{1}/{p}}},   \quad ~ \forall m=1,\ldots,M.
  \end{equation}
\end{proposition}

\begin{proof}\footnote{We remark that a more general result can be obtained by an alternative proof using H\"older's inequality \cite[see Lemma 26 in][]{MicPon05}.}
We start the derivation, by 
equivalently translating \refprimal via Theorem~\ref{th:reg-obj} into
\begin{equation}\label{varma_again}
  \inf_{\w, b,\vtheta:\vtheta\geq{\bf 0}} \quad  \tilde{C}\sum_{i=1}^n \loss\left(\sum_{m=1}^M \langle \w_m,\psi_m(\x_i)\rangle _{\mathcal{H}_m}+b,~y_i\right)+ \frac{1}{2}\sum_{m=1}^M\frac{\Vert\w_m\Vert_{\mathcal{H}_m}^2}{\theta_m} + \frac{\mu}{2}\Vert\vtheta\Vert_p^2,
\end{equation}
with $\mu>0$. Suppose we are given fixed $(\w,b)$, then setting the partial derivatives of the above objective w.r.t.~$\vtheta$ to zero yields
the following condition on the optimality of $\vtheta$,
\begin{equation}\label{eq:direct_op}
   -\frac{\Vert\w_m\Vert_{\mathcal{H}_m}^2}{2\theta_m^2} + \mu\cdot\frac{\partial \left(\frac{1}{2}\Vert\vtheta\Vert_p^2\right)}{\partial\theta_m}=0,   \quad \forall m=1,\ldots,M.
\end{equation}
The first derivative of the $\ell_p$-norm with respect to the mixing coefficients can be expressed as 
$$\frac{\partial\left(\frac{1}{2}\Vert\vtheta\Vert_p^2\right)}{\partial\theta_m}=\theta_m^{p-1}\Vert\vtheta\Vert_p^{2-p} ,$$
and hence Eq. \eqref{eq:direct_op} translates into the following optimality condition,
\begin{equation}\label{eq:op_temp}
   \exists\zeta \quad \forall m=1,\ldots,M: ~~ \quad \theta_m = \zeta\Vert\w_m\Vert_{\mathcal{H}_m}^{\frac{2}{p+1}}~. \quad \quad
\end{equation}
 Because $\w\neq0$, using the same argument as in the proof of Theorem \ref{th:reg-obj}, the constraint $\Vert\vtheta\Vert^2_p\leq 1$ in \eqref{varma_again} is at the upper bound, i.e.\  $\Vert\vtheta\Vert_p=1$ holds for an optimal $\vtheta$. Inserting \eqref{eq:op_temp} in the latter equation leads to $\zeta = \left(\sum_{m=1}^M\Vert\w_m\Vert_{\mathcal{H}_m}^{{2p}/{p+1}}\right)^{{1}/{p}}$. Resubstitution into \eqref{eq:op_temp} yields the  claimed formula \eqref{directBetasOpt}.
\end{proof}

%
%
Second, we consider how to optimize \refprimal w.r.t.~the
remaining variables $(\w,b)$ for a given set
of mixing coefficients $\vtheta$. Since optimization often is considerably easier in the dual space, we fix $\vtheta$ and
build the partial Lagrangian of \refprimal w.r.t.~all other primal variables $\w$, $b$. 
The resulting dual problem is of the form (detailed derivations omitted)
\begin{align}\label{regularSVM}
  \sup_{\valpha:\one^\top\valpha=0} 
  & -\C\sum_{i=1}^n \loss^*\left(-\frac{\alpha_i}{\C},~y_i\right)  -\frac{1}{2}\sum_{m=1}^M \theta_m\valpha^\top K_m \valpha,
\end{align}
\blue{and the KKT conditions yield $\vec w_m = \theta_m\sum_{i=1}^n \alpha_i \psi_m(\vec x_i)$ in the optimal point, hence
\begin{equation}\label{eq_kkt-temp}
  \Vert\vec w_m\Vert^2 = \theta_m^2\valpha K_m \valpha, \quad \forall~m=1,\ldots,M.
\end{equation}  
}
We now have all ingredients (i.e., Eqs.~\eqref{directBetasOpt}, \eqref{regularSVM}--\eqref{eq_kkt-temp}) to formulate a simple macro-wrapper algorithm for $\ell_p$-norm MKL training: 

\begin{algorithm*}[hbt]
  \begin{algorithmic}[1]
  \small
  \STATE \textbf{input:} feasible $\valpha$ and $\vtheta$
  \STATE \textbf{while} optimality conditions are not satisfied {\bf do}
  \STATE \qquad Compute $\valpha$ according to  Eq.~\eqref{regularSVM}  (e.g.\ SVM)
  \blue{
  \STATE \qquad Compute $\Vert\w_m\Vert^2$ for all $m=1,...,M$ according to Eq.~\eqref{eq_kkt-temp} 
  \STATE \qquad Update $\vtheta$ according to Eq.~\eqref{directBetasOpt} 
  }
  \STATE \textbf{end while}
  \end{algorithmic}
  \caption{\label{alg:direct_wrapper} \textit{Simple $\ell_{p>1}$-norm MKL wrapper-based training algorithm.
  The analytical updates of $\vtheta$ and the SVM computations are optimized alternatingly.}}
\end{algorithm*}

The above algorithm alternatingly solves a convex risk minimization machine (e.g.\ SVM) w.r.t.~the actual mixture $\vtheta$  (Eq.~\eqref{regularSVM}) and
subsequently computes the analytical  update according to Eq.~\eqref{directBetasOpt} and \eqref{eq_kkt-temp}. It can, for example, be stopped based on changes of the objective function or the duality gap within subsequent iterations.

\subsection{Towards Large-Scale MKL---Interleaving SVM and MKL Optimization}\label{analytical}

However, a disadvantage of the above wrapper approach still is that it deploys a full blown kernel matrix. We thus
propose to interleave the SVM optimization of SVMlight with the $\vtheta$- and $\valpha$-steps at training time.
We have implemented this so-called \emph{interleaved} algorithm in Shogun for hinge loss,
thereby promoting sparse solutions in $\alpha$. 
This allows us to solely operate on a small number of active variables.\footnote{In practice, it turns out that 
the kernel matrix of active variables typically is about of the size $40\times 40$, even when we deal with ten-thousands of examples.}
The resulting interleaved optimization method is shown in Algorithm~\ref{alg:directmkl}.
Lines 3-5 are standard in chunking based SVM solvers and carried out
by $\text{SVM}^{\text{light}}$ \blue{(note that $Q$ is chosen as described in \cite{Joa99})}. Lines 6-7 compute SVM-objective
values. Finally,  the analytical 
$\vtheta$-step is carried out in Line 9. The algorithm  terminates if the
maximal KKT violation \citep[c.f.][]{Joa99} falls below a predetermined precision
$\varepsilon$ and if the normalized maximal constraint violation
$|1-\frac{\omega}{\omega_{\text old}}|<\varepsilon_{mkl}$ for the MKL-step, where $\omega$ denotes the MKL objective function value (Line 8).

\begin{algorithm*}[hbt]
  \begin{algorithmic}[1]
    \small
	\STATE \textbf{Initialize:} $g_{m,i}=\hat g_i=\alpha_i=0$, $\forall i=1,...,n$;  \quad $L=S=-\infty$; \quad 
	$\theta_m=\sqrt[p]{1/M}$, $\forall m=1,...,M$
	\STATE \textbf{iterate}
	\STATE \quad Select Q variables $\alpha_{i_1},\ldots,\alpha_{i_Q}$ based on the gradient $\hat {\bf g}$ of \eqref{regularSVM} w.r.t.\ $\valpha$
	\STATE \quad Store $\valpha^{old}=\valpha$ and then update $\valpha$\vspace*{0.2ex} according to \eqref{regularSVM} with respect to the selected
	variables 
    \STATE \quad Update gradient $g_{m,i}\leftarrow g_{m,i}+ \sum_{q=1}^Q (\alpha_{i_q}-\alpha^{old}_{i_q})
	    \k_m(\x_{i_q},\x_i)$, ~$\forall$ $m=1,\ldots,M$,	$i=1,\ldots,n$ \vspace{0.25ex}
  \STATE \quad Compute the quadratic terms ~$S_m=\frac{1}{2}\sum_{i} g_{m,i}\alpha_i$, ~$q_m = 2\theta_m^2 S_m$, $~\forall m=1,\ldots,M$ \vspace{0.5ex}
  \STATE \quad $L_{\text old}=L$, \quad $L=\sum_iy_i\alpha_i$, \quad $S_{\text old}=S$, \quad $S=\sum_{m} \theta_m S_m$  \vspace{0.4ex}
	\STATE \quad {\bf if} $|1-\frac{L-S}{L_{\text old}-S_{\text old}}|\geq\varepsilon$  
  \STATE \quad  \quad $ \theta_m = \left(q_{m}\right)^{1/(p+1)}  / \left(\sum_{m'=1}^M\left(q_{m'}\right)^{p/(p+1)}\right)^{1/p}, \quad \forall~ m=1,\ldots,M$
  \STATE \quad {\bf else}
  \STATE \qquad \textbf{break}
  \STATE \quad {\bf end if}
  \STATE \quad $\hat g_i=\sum_m \theta_m g_{m,i}$ for all $i=1,\ldots,n$
  \end{algorithmic}
  \caption{\label{alg:directmkl} ~  \textit{$\ell_p$-Norm MKL chunking-based training
  algorithm via analytical update. Kernel weighting $\vtheta$ and (signed) SVM $\valpha$  are optimized interleavingly. The
    accuracy parameter $\varepsilon$ and the subproblem size $Q$ are
    assumed to be given to the algorithm.}}
\end{algorithm*}

\subsection{Convergence Proof for $p>1$}

In the following, we exploit the primal view of the above algorithm as a nonlinear block Gauss-Seidel method,
to prove convergence of our algorithms. We first need the following useful result about convergence of the nonlinear block Gauss-Seidel method in general.

\begin{proposition}[{Bertsekas, 1999, Prop. 2.7.1}]
	\label{prop:gauss-seidel}
  Let $\mathcal X=\bigotimes_{m=1}^M\mathcal{X}_m$ be the Cartesian product of closed convex sets $\mathcal{X}_m\subset\mathbb R^{d_m}$, be $f:\mathcal X\rightarrow \mathbb R$ a continuously differentiable function. Define the nonlinear block Gauss-Seidel method recursively by letting $\x^0\in \mathcal X$ be any feasible point, and be
  \begin{equation}\label{eq:seidel_method}
     \x_m^{k+1}=\argmin_{\boldsymbol{\xi}\in\mathcal X_m} f\left(\x_1^{k+1},\cdots,\x_{m-1}^{k+1},\boldsymbol{\xi},\x_{m+1}^{k},\cdots,\x_{M}^{k}\right), ~ ~ \forall m=1,\ldots,M.
   \end{equation}     
  Suppose that for each $m$ and $\x\in\mathcal X$, the minimum 
  \begin{equation}\label{eq:unique}
    \min_{\boldsymbol{\xi}\in\mathcal X_m} f\left(\x_1,\cdots,\x_{m-1},\boldsymbol{\xi},\x_{m+1},\cdots,\x_M\right)
  \end{equation}
  is uniquely attained. 
  Then every limit point of the sequence $\{\x^k\}_{k\in\mathbb N}$ is a stationary point.
\end{proposition}
The proof can be found in \cite{Ber99}, p.~268-269. The next proposition basically establishes convergence of the proposed $\ell_p$-norm MKL training algorithm.  

\begin{theorem}
	\label{thm:directmkl}
  Let $\loss$ be the hinge loss and be $p> 1$. Let the kernel matrices $K_1,\ldots,K_M$ be positive definite.
  Then every limit point of   Algorithm~\ref{alg:direct_wrapper} is a globally optimal point of \refprimal.
  Moreover, suppose that the SVM computation is solved exactly in each iteration, then the same holds true for Algorithm~\ref{alg:directmkl}.
\end{theorem}

\begin{proof}
If we ignore the numerical speed-ups, 
then the Algorithms \ref{alg:direct_wrapper} and \ref{alg:directmkl} coincidence for the hinge loss. Hence, 
it suffices to show the wrapper algorithm converges.

To this aim, we have to transform \refprimal into a form such that the requirements for application of Prop.~\ref{prop:gauss-seidel} are fulfilled.
We start by expanding \refprimal into
\begin{align*}
     \min_{\w, b,\vxi,\vtheta} & \quad  C\sum_{i=1}^n \xi_i+ \frac{1}{2}\sum_{m=1}^M\frac{\Vert\w_m\Vert_{\mathcal{H}_m}^2}{\theta_m}, \\ \text{s.t.} ~ &\quad \forall i: ~ \sum_{m=1}^M\langle\w_m,\psi_m(\x_i)\rangle _{\mathcal{H}_m}+b\geq 1-\xi_i; \quad \vxi\geq 0;  \quad \Vert\vtheta\Vert_p^2 \leq 1; \quad \vtheta\geq {\bf 0},
\end{align*}
thereby extending the second block of variables, $(\w,b)$, into $(\w,b,\vxi)$. Moreover, we note
that after an application of the representer theorem\footnote{Note that the coordinate transformation into $\mathbb R^n$ can be explicitly given in terms of the empirical kernel map \citep{SchMikBurKniMueRaeSmo99}.} \citep{KiWa71} we may without loss of generality assume $\mathcal H_m= \mathbb R^{n}$.

In the problem's current form, the possibility of $\theta_m=0$ while $\w_m\neq 0$ renders the objective function nondifferentiable. 
This hinders the application of Prop.~\ref{prop:gauss-seidel}.
Fortunately, it follows from Prop.~\ref{prop:directBetasOpt} (note that $K_m\succ 0$ implies $\w\neq \zero$) that this case is impossible.
We therefore can substitute the constraint $\vtheta\geq \zero$ by $\vtheta >\zero$ for all $m$. In order to maintain the closeness of the feasible set
we subsequently apply a bijective coordinate transformation $\phi:\mathbb R^M_+\rightarrow\mathbb R^M$ with $\theta_m^{\text{new}}=\phi_m(\theta_m)=\log(\theta_m)$, resulting in the following equivalent problem,
\begin{align*}
     \inf_{\w, b,\vxi,\vtheta} & \quad  C\sum_{i=1}^n \xi_i+ \frac{1}{2}\sum_{m=1}^M\exp(-\theta_m)\Vert\w_m\Vert_{\mathbb R^n}^2, \\ \text{s.t.} ~ &\quad \forall i: ~ \sum_{m=1}^M\langle\w_m,\psi_m(\x_i)\rangle _{\mathbb R^n}+b\geq 1-\xi_i; \quad \vxi\geq 0;  \quad \Vert\exp(\vtheta)\Vert_p^2 \leq 1,
\end{align*}
where we employ the notation $\exp(\vtheta)=\left(\exp(\theta_1),\cdots,\exp(\theta_M)\right)^\top$. 

Applying the Gauss-Seidel method in Eq.~\eqref{eq:seidel_method} to the base problem \refprimalcounter and to the reparametrized problem yields the same sequence of solutions $\{(\w,b,\vtheta)^k\}_{k\in\mathbb N_0}$. The above problem now allows to apply Prop.~\ref{prop:gauss-seidel} for the two blocks of coordinates $\vtheta\in\mathcal X_1$ and $(\w,b,\vxi)\in\mathcal X_2$: the objective is continuously differentiable and the sets $\mathcal X_1$ are closed and convex. To see the latter, note that $\Vert\cdot\Vert_p^2\circ\exp$ is a convex function, since $\Vert\cdot\Vert_p^2$ is convex and non-increasing in each argument \cite[cf., e.g., Section 3.2.4 in][] {BoyVan04}. Moreover, the minima in Eq.~\eqref{eq:seidel_method} are uniquely attained: the $(\w,b)$-step amounts to solving an SVM on a positive definite kernel mixture, and the analytical $\vtheta$-step clearly yields unique solutions as well.

Hence, we conclude that every limit point of the sequence $\{(\w,b,\vtheta)^k\}_{k\in\mathbb N}$ is a stationary point of \refprimal. For a convex problem, this is equivalent to such a limit point being globally optimal.
\end{proof}

In practice, we are facing two problems. First, 
the standard Hilbert space setup necessarily implies that
$\norm{\w_m} \geq 0$ for all $m$.  However in practice this assumption
may often be violated, either due to numerical imprecision or because
of using an indefinite ``kernel'' function. However, for any
$\norm{\w_m}\leq 0$ it also follows that $\theta^\star_m=0$ as long as at least one strictly positive 
$\norm{\w_{m'}}> 0$ exists.
This is because for any $\lambda< 0$ we have
$\lim_{h\rightarrow0,h> 0} \frac{\lambda}{h} = -\infty$.
Thus, for any $m$ with $\norm{\w_m}\leq 0$, we can immediately set
the corresponding mixing coefficients $\theta^\star_m$ to zero.
The remaining $\vtheta$ are then computed according to Equation \eqref{alg:directmkl},
and convergence will be achieved as long as at least one strictly positive 
$\norm{\w_{m'}}> 0$ exists in each iteration.


Second, in practice, the SVM problem will only be solved with finite precision,
which may lead to convergence problems.
Moreover, we actually want to improve the $\valpha$ only a little 
bit before recomputing $\vtheta$ since 
computing a high precision solution can be wasteful, as indicated by the superior performance of 
the interleaved algorithms (cf. Sect.~\ref{sec-exectime}).
 This helps to avoid spending a lot of $\valpha$-optimization (SVM training) on
a suboptimal mixture $\vtheta$.
Fortunately, we can overcome the potential convergence problem by ensuring that the primal objective decreases 
within each $\valpha$-step. 
This is enforced in practice, by computing the SVM by a higher precision if needed.
However, in our computational experiments we find that this
precaution is not even necessary: even without it, the algorithm
converges in all cases that we tried (cf. Section \ref{SEC-experiments}).


Finally, we would like to point out that the proposed block coordinate descent approach lends itself more
naturally to combination with primal SVM optimizers like
\citep{Chapelle2006}, LibLinear \citep{FanChaHsiWanLin08} or Ocas \citep{FraSon08}.  Especially for
linear kernels this is extremely appealing.

\subsection{Technical Considerations}

\subsubsection{Implementation Details}
\label{sec:implementation}

We have implemented the analytic optimization algorithm described in the previous Section, as well as the cutting plane  and 
Newton algorithms by \cite{KloBreSonZieLasMue09}, within the SHOGUN
toolbox \citep{shogun} for regression,
one-class classification, and two-class classification tasks. In addition one
can choose the optimization scheme, i.e., decide whether the interleaved
optimization algorithm or the wrapper algorithm should be applied.
In all approaches any of the SVMs contained in SHOGUN can be used. 
Our implementation can be downloaded from \url{http://www.shogun-toolbox.org}.

In the more conventional family of approaches, the \emph{wrapper algorithms},
an optimization scheme on $\vtheta$ wraps around a single kernel SVM.
Effectively this results in alternatingly solving for $\valpha$ and $\vtheta$.
For the outer optimization (i.e., that on $\vtheta$) SHOGUN offers the three choices listed above.
The semi-infinite program (SIP) uses a traditional SVM to generate
new violated constraints and thus requires a single kernel SVM.
A linear program (for $p=1$) or a sequence of quadratically
constrained linear programs (for $p>1$) is solved via
GLPK\footnote{\url{http://www.gnu.org/software/glpk/}.} or IBM ILOG
CPLEX\footnote{\url{http://www.ibm.com/software/integration/optimization/cplex/}.}.
Alternatively, either an analytic or a Newton update (for $\ell_p$ norms with $p>1$)
step can be performed, obviating the need for an additional mathematical programming software.

The second, much faster approach performs interleaved optimization and thus requires
modification of the core SVM optimization algorithm. It is currently 
integrated into the chunking-based SVRlight and SVMlight.
To reduce the implementation effort, we implement a single
function \texttt{perform\_mkl\_step($\sum_{\valpha}$, obj$_m$)}, that has the arguments
$\sum_{\valpha}=\sum_{i=1}^n \alpha_i$ and obj$_m$=$\frac{1}{2}\valpha^T K_m \valpha$,
i.e.\ the current linear $\valpha$-term and the SVM objectives for each kernel.
This function is either, in the interleaved optimization case, called
as a callback function (after each chunking step or a couple of SMO
steps), or it is called by the wrapper algorithm (after each SVM
optimization to full precision).

\paragraph{Recovering Regression and One-Class Classification.}
It should be noted that one-class classification is trivially implemented using
$\sum_{\valpha}=0$ while support vector regression (SVR) is typically performed by internally
translating the SVR problem into a standard SVM classification problem with
twice the number of examples once positively and once negatively labeled with
corresponding $\valpha$ and $\valpha^*$. Thus one needs direct access to
$\valpha^*$ and computes $\sum_{\valpha}=-\sum_{i=1}^{n}
(\alpha_i+\alpha_i^*)\varepsilon - \sum_{i=1}^{n} (\alpha_i - \alpha_i^*)y_i$
\cite[cf.][]{SonRaeSchSch06}. Since this requires modification of the core SVM
solver we implemented SVR only for interleaved optimization and SVMlight.

\paragraph{Efficiency Considerations and Kernel Caching.}
Note that the choice of the size of the kernel cache becomes crucial when
applying MKL to large scale learning applications.\footnote{\emph{Large scale} in the sense, that the data cannot be stored in memory or the computation reaches a maintainable limit. In the case of MKL this can be due both a large sample size or a high number of kernels.}
 While for the wrapper
algorithms only a \emph{single} kernel SVM needs to be solved and thus a single
large kernel cache should be used, the story is different for interleaved
optimization. Since one must keep track of the several partial MKL
objectives obj$_m$, requiring access to individual kernel rows, the same
cache size should be used for all sub-kernels.


\subsubsection{Kernel Normalization}
\label{sec:kernelnormalization}

The normalization of kernels is as important for MKL as the
normalization of features is for training regularized linear or
single-kernel models.  This is owed to the bias introduced by the
regularization: optimal feature / kernel weights are requested to be
small.  This is easier to achieve for features (or entire feature
spaces, as implied by kernels) that are scaled to be of large
magnitude, while downscaling them would require a correspondingly
upscaled weight for representing the same predictive model.  Upscaling
(downscaling) features is thus equivalent to modifying regularizers
such that they penalize those features less (more).  As is common
practice, we here use isotropic regularizers, which penalize
all dimensions uniformly.  This implies that the kernels have to
be normalized in a sensible way in order to represent an
``uninformative prior'' as to which kernels are useful.

There exist several approaches to kernel normalization, of which we
use two in the computational experiments below.  They are
fundamentally different. The first one generalizes the common
practice of standardizing features to entire kernels, thereby directly
implementing the spirit of the discussion above. 
\blue{ In contrast, the
second normalization approach rescales the data points to unit norm in feature space. }
Nevertheless it can have a
beneficial effect on the scaling of kernels, as we argue below.

\paragraph{Multiplicative Normalization.}
As done in \cite{OngZien08}, we multiplicatively normalize the kernels to
have uniform variance of data points in feature space. Formally, we
find a positive rescaling $\rho_m$ of the kernel, such that the
rescaled kernel $\tilde k_m(\cdot,\cdot) = \rho_m k_m(\cdot,\cdot)$ and
the corresponding feature map $\tilde\Phi_m(\cdot) = \sqrt{\rho_m}
\Phi_m(\cdot)$ satisfy
\blue{
\begin{eqnarray*}
   \frac{1}{n} \sum_{i=1}^n \left\Vert \tilde\Phi_m(\x_i) - \tilde\Phi_m(\bar\x) \right\Vert^2 = 1 
\end{eqnarray*}
for each $m = 1,\ldots,M$, where $\tilde\Phi_m(\bar\x) := \frac{1}{n} \sum_{i=1}^n \tilde\Phi_m(\x_i)$
is the empirical mean of the data in feature space.
The above equation can be  equivalently be expressed in terms of kernel functions as
\begin{eqnarray*}
  \frac{1}{n} \sum_{i=1}^n \tilde k_m(\x_i,\x_i) - \frac{1}{n^2} \sum_{i=1}^n \sum_{j=1}^n \tilde k_m(\x_i,\x_j)=1,
\end{eqnarray*}
}
so that the final normalization rule is
 \begin{equation}\label{normvariance}
\k(\x,\bar\x) \longmapsto \frac{k(\x,\bar\x)}{\frac{1}{n}\sum_{i=1}^n k(\x_i, \x_i) - \frac{1}{n^2}\sum_{i,j=1}^n, k(\x_i, \x_j)}.
\end{equation}
\blue{Note that in case the kernel is centered (i.e.\ the empirical mean of the data points lies on
the origin), the above rule simplifies to 
$ \k(\x,\bar\x) \longmapsto {\k(\x,\bar{\x})}/{\frac{1}{n}{\rm tr}(K)}$,
where ${\rm tr}(K):=\sum_{i=1}^n \k(\x_i, \x_i)$ is the trace of the kernel matrix $K$.
}

\paragraph{Spherical Normalization.}
Frequently, kernels are normalized according to
\begin{align}\label{norm}
\k(\x,\bar\x)\longmapsto\frac{\k(\x,\bar\x)}{\sqrt{ k(\x,\x)k(\bar\x,\bar\x)}}.
\end{align}
After this operation, $\|\x\|=k(\x,\x)=1$ holds for each data point $x$;
this means that each data point is rescaled to lie on the unit sphere.
\blue{Still, this also may have an effect on the scale of the features: a
spherically normalized and centered kernel is also always  multiplicatively 
normalized, because the multiplicative normalization rule becomes
$ \k(\x,\bar\x) \longmapsto {\k(\x,\bar\x)}/{\frac{1}{n}{\rm tr}(K)} = {\k(\x,\bar\x)}/1$.
}

Thus the spherical
normalization may be seen as an approximation to the above
multiplicative normalization and may be used as a substitute for it.
Note, however, that it changes the data points themselves by eliminating
length information; whether this is desired or not depends on the
learning task at hand. Finally note that both normalizations achieve that the
optimal value of $C$ is not far from $1$.

\subsection{Limitations and Extensions of our Framework}
\label{section-limits}
In this section, we  show the connection of $\ell_p$-norm MKL to 
a formulation based on block norms, point out  
limitations and sketch extensions of our framework.
To this aim let us recall the primal MKL problem \refprimalcounter and consider the special case of $\ell_p$-norm MKL given by
\begin{align}\label{eq:lpmkl}
\hspace{-0.4cm}
     \inf_{\w, b,\vtheta:\vtheta\geq {\bf 0}}  \quad  C\sum_{i=1}^n \loss\left(\sum_{m=1}^M \langle \w_m,\psi_m(\x_i)\rangle _{\mathcal{H}_m}+b,~y_i\right)+ \frac{1}{2}\sum_{m=1}^M\frac{\Vert\w_m\Vert_{\mathcal{H}_m}^2}{\theta_m}, \quad \text{s.t.} \quad   \Vert\vtheta\Vert_p^2 \leq 1. 
\end{align}
The subsequent proposition shows that \eqref{eq:lpmkl} equivalently can be translated into the following mixed-norm formulation,
\begin{align}\label{eq:block-norm}
     \inf_{\w, b}  \quad  \tilde{C}\sum_{i=1}^n \loss\left(\sum_{m=1}^M \langle \w_m,\psi_m(\x_i)\rangle _{\mathcal{H}_m}+b,~y_i\right)+  \frac{1}{2}\sum_{m=1}^M \Vert\w_m\Vert_{\mathcal{H}_m}^q ,
\end{align}
where $q=\frac{2p}{p+1}$, and $\tilde{C}$ is a constant. This has been studied by \cite{BacLanJor04} for $q=1$ and by \cite{SzaGraRak08} for hierarchical penalization. 

\begin{proposition}\label{prop:block-norm}
  Let be $p> 1$, be $\loss$ a convex loss function, and define $q:=\frac{2p}{p+1}$ (i.e.\ $p=\frac{q}{2-q}$). Optimization Problem
  \eqref{eq:lpmkl} and \eqref{eq:block-norm} are equivalent, i.e., for each $C$ there exists a $\tilde{C}> 0$, such that for each optimal solution ($\w^*,b^*,\theta^*$) of OP~\eqref{eq:lpmkl} using $C$, we have that ($\w^*,b^*$) is also optimal in OP~\eqref{eq:block-norm} using $\tilde{C}$, and vice versa.
\end{proposition}

\begin{proof}
\blue{
From Prop.~\ref{prop:directBetasOpt} it follows that for any fixed $\w$  in \eqref{eq:lpmkl} it holds for the $\w$-optimal $\vtheta$:
\begin{equation*}
      \exists \zeta : ~ ~ \theta_m = \zeta\Vert\w_m\Vert_{\mathcal{H}_m}^{\frac{2}{p+1}},   \quad \forall m=1,\ldots,M.
\end{equation*}
Plugging the above equation into \eqref{eq:lpmkl} yields 
\begin{align}
     \inf_{\w, b}  \quad  C\sum_{i=1}^n \loss\left(\sum_{m=1}^M \langle \w_m,\psi_m(\x_i)\rangle _{\mathcal{H}_m}+b,~y_i\right)+ \frac{1}{2\zeta}\sum_{m=1}^M \Vert\w_m\Vert_{\mathcal{H}_m}^{\frac{2p}{p+1}} .
\end{align}
Defining $q:=\frac{2p}{p+1}$ and $\tilde{C} := \zeta C$ results in \eqref{eq:block-norm}.
}
\end{proof}

Now, let us take a closer look on the parameter range of 
$q$. It is easy to see that when we vary $p$ in the real interval $[1,\infty]$, then $q$ is limited to range in $[1,2]$. 
So in other words the methodology presented in this paper only covers the $1\leq q \leq 2$ block norm case.
\blue{
However, from an algorithmic perspective our framework can be easily extended to the $q>2$ case:  
although originally aiming at the more sophisticated case of hierarchical kernel learning,
\cite{AflBenBhaNatRamJMLR} showed in particular that for $q\geq 2$, Eq.~\eqref{eq:block-norm} is equivalent to 
\begin{align}\label{eq:mixednorm}
     \sup_{\vtheta:\vtheta\geq {\bf 0},\Vert\vtheta\Vert^2_{\p}\leq 1} ~ \inf_{\w, b}  \quad  \tilde{C}\sum_{i=1}^n \loss\left(\sum_{m=1}^M \langle \w_m,\psi_m(\x_i)\rangle _{\mathcal{H}_m}+b,~y_i\right)+  \frac{1}{2}\sum_{m=1}^M\theta_m\Vert\w_m\Vert_{\mathcal{H}_m}^2 ,
\end{align}
where $\p:=\frac{q}{q-2}$.}
Note the difference to $\ell_p$-norm MKL:
the mixing coefficients $\vtheta$ appear in the nominator 
and by varying $\p$ in the interval 
$[1,\infty]$, the range of $q$ in the interval $[2,\infty]$
can be obtained, which explains why this method is complementary 
to ours, where $q$ ranges in $[1,2]$. 

\blue{
It is straight forward to show that for every fixed (possibly suboptimal) pair $(\w,b)$ the optimal $\vtheta$ is given by
$$ \theta_m = \frac{\Vert\w_m\Vert_{\mathcal{H}_m}^{\frac{2}{r-1}}}{\left(\sum_{m'=1}^M\Vert\w_{m'}\Vert_{\mathcal{H}_{m'}}^{\frac{2r}{r-1}}\right)^{{1/r}}},   
  \quad ~ \forall m=1,\ldots,M.$$
The proof is analogous to that of Prop.~\ref{prop:directBetasOpt} and the above analytical update formula can be
used to derive a block coordinate descent algorithm that is analogous to ours.
In our framework,  the mixings $\theta$, however, appear in the denominator of the objective function of \refprimal. Therefore, the corresponding update formula in our framework is 
\begin{equation}\label{direct_temp}
   \theta_m = \frac{\Vert\w_m\Vert_{\mathcal{H}_m}^{\frac{-2}{r-1}}}{\left(\sum_{m'=1}^M\Vert\w_{m'}\Vert_{\mathcal{H}_{m'}}^{\frac{-2r}{r-1}}\right)^{{1/r}}},   
  \quad ~ \forall m=1,\ldots,M.
\end{equation}
This shows that
we can simply optimize $2<q\leq\infty$-block-norm MKL within our 
computational framework, using the update formula \eqref{direct_temp}.
}



\section{Computational Experiments}\label{SEC-experiments}

In this section we study non-sparse MKL in terms of computational
efficiency and predictive accuracy. 
We apply the method of \cite{SonRaeSchSch06} in the case of $p=1$.
We write $\ell_\infty$-norm MKL for a regular SVM with the unweighted-sum
kernel $K=\sum_m K_m$.

We first study a toy problem in Section 
\ref{sectoy1} where we have full control over the distribution of the relevant 
information in order to shed light on the appropriateness of sparse, non-sparse, and
$\ell_\infty$-MKL. We report on 
real-world problems from bioinformatics, namely
protein subcellular localization (Section \ref{sec:prot}), 
 finding transcription start sites of
RNA Polymerase II binding genes in genomic DNA sequences (Section
\ref{sec:tss}), and  
reconstructing metabolic gene networks (Section \ref{sec:bleakley}).


\subsection{Measuring the Impact of Data Sparsity---Toy Experiment}\label{sectoy1}

\newcommand{\truevtheta}{{\vtheta}}
\newcommand{\truethetai}{{\theta_i}}
\newcommand{\mklvtheta}{{\widehat{\vtheta}}}

The goal of this section is to study the relationship of the level of sparsity
of the true underlying function to be learned to the chosen norm $p$ in the model.
\blue{Intuitively, we might expect that the optimal choice of $p$ directly corresponds to
the true level of sparsity.} Apart from verifying this conjecture,
we are also interested in the effects of suboptimal choice of $p$.
To this aim we constructed several artificial data sets in which we vary 
the degree of sparsity in the true kernel mixture coefficients.
We go from having all weight focussed on a single kernel (the highest level of sparsity)
to uniform weights (the least sparse scenario possible) in several steps.
We then study the statistical performance of $\ell_p$-norm MKL for different
values of $p$ that cover the entire range $[1,\infty]$.
\begin{figure}[t]
  \centering
    \includegraphics[width=0.5\textwidth]{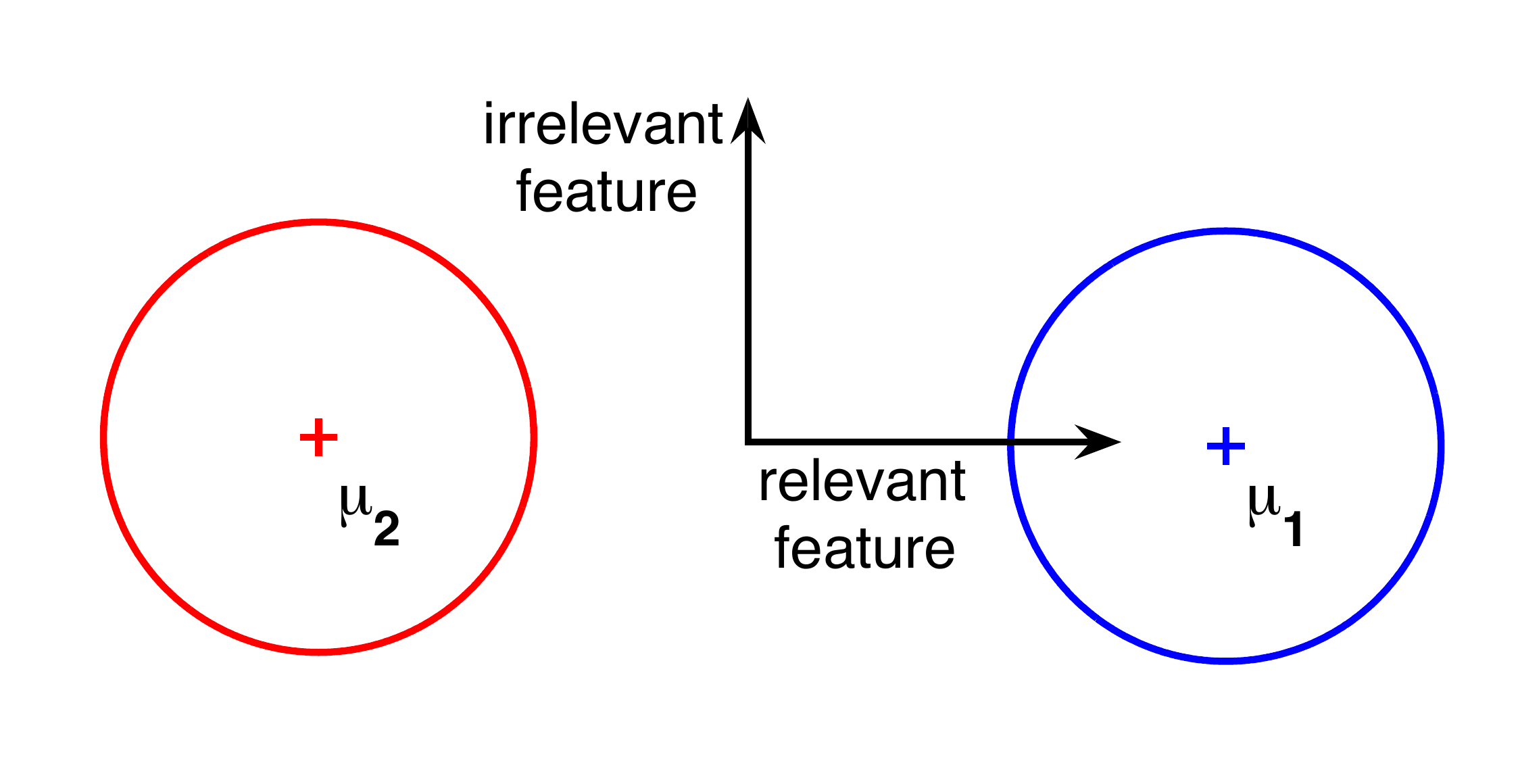}
    \caption{\blue{Illustration of the toy experiment for $\truevtheta=(1,0)^\top$.\label{fig:illustr_toy}}}
\end{figure}

We generate an $n$-element balanced sample
$\mathcal{D}=\{(\x_i,y_i)\}_{i=1}^n$ from two $d=50$-dimensional
isotropic Gaussian distributions with equal covariance matrices
$C=I_{d\times d}$ and equal, but opposite, means 
$\mu_1=\frac{\rho}{\Vert\truevtheta\Vert_2}\truevtheta$ and
$\mu_2=-\mu_1$.  Thereby $\truevtheta$ is a
binary vector, i.e., $\forall i:\truethetai\in\{0,1\}$, encoding the true underlying
data sparsity as follows. Zero components $\truethetai=0$ clearly imply
identical means of the two classes' distributions in the $i$th feature
set; hence the latter does not carry any discriminating information.
In summary, the fraction of zero components,
$\nu(\truevtheta)=1-\frac{1}{d}\sum_{i=1}^d \truethetai$, is a measure for
the feature sparsity of the learning problem.

For several values of $\nu$ we generate $m=250$ data sets $\mathcal{D}_1,\ldots,\mathcal{D}_m$ fixing $\rho=1.75$. 
Then, each feature is input to a linear kernel and the resulting kernel matrices are multiplicatively normalized as described
in Section \ref{sec:kernelnormalization}.
Hence, $\nu(\truevtheta)$ gives the fraction of noise kernels in the working kernel set.
Then, classification models are computed by training $\ell_p$-norm MKL for $p=1,4/3,2,4,\infty$ on each $\mathcal D_i$. Soft  margin parameters $C$ are tuned on independent $10,000$-elemental validation sets by grid search over $C\in 10^{[-4,3.5,\ldots,0]}$ (optimal $C$s are attained in the interior of the grid). 
The relative duality gaps were optimized up to a precision of $10^{-3}$.
We report on test errors evaluated on $10,000$-elemental independent test sets and
 pure mean $\ell_2$ model errors of the computed kernel mixtures, that is ${\rm ME}(\mklvtheta)=\Vert\zeta(\mklvtheta)-\zeta(\truevtheta)\Vert_2$, where $\zeta(\x)=\frac{\x}{\Vert\x\Vert_2}$.

The results are shown in Fig.~\ref{toy-gap50}  for $n=50$ and $n=800$, where the figures on the left show the test errors and the ones on the right the model errors  ${\rm ME}(\mklvtheta)$. Regarding the latter, model errors reflect the corresponding test errors for $n=50$. This observation can be explained by 
statistical learning theory. 
The minimizer of the empirical risk performs unstable for small sample sizes 
and the model selection results in a strongly regularized hypothesis, 
leading to the observed agreement between test error and model error.

\begin{figure}[t]
  \subfigure[]{
    \includegraphics[width=0.49\textwidth]{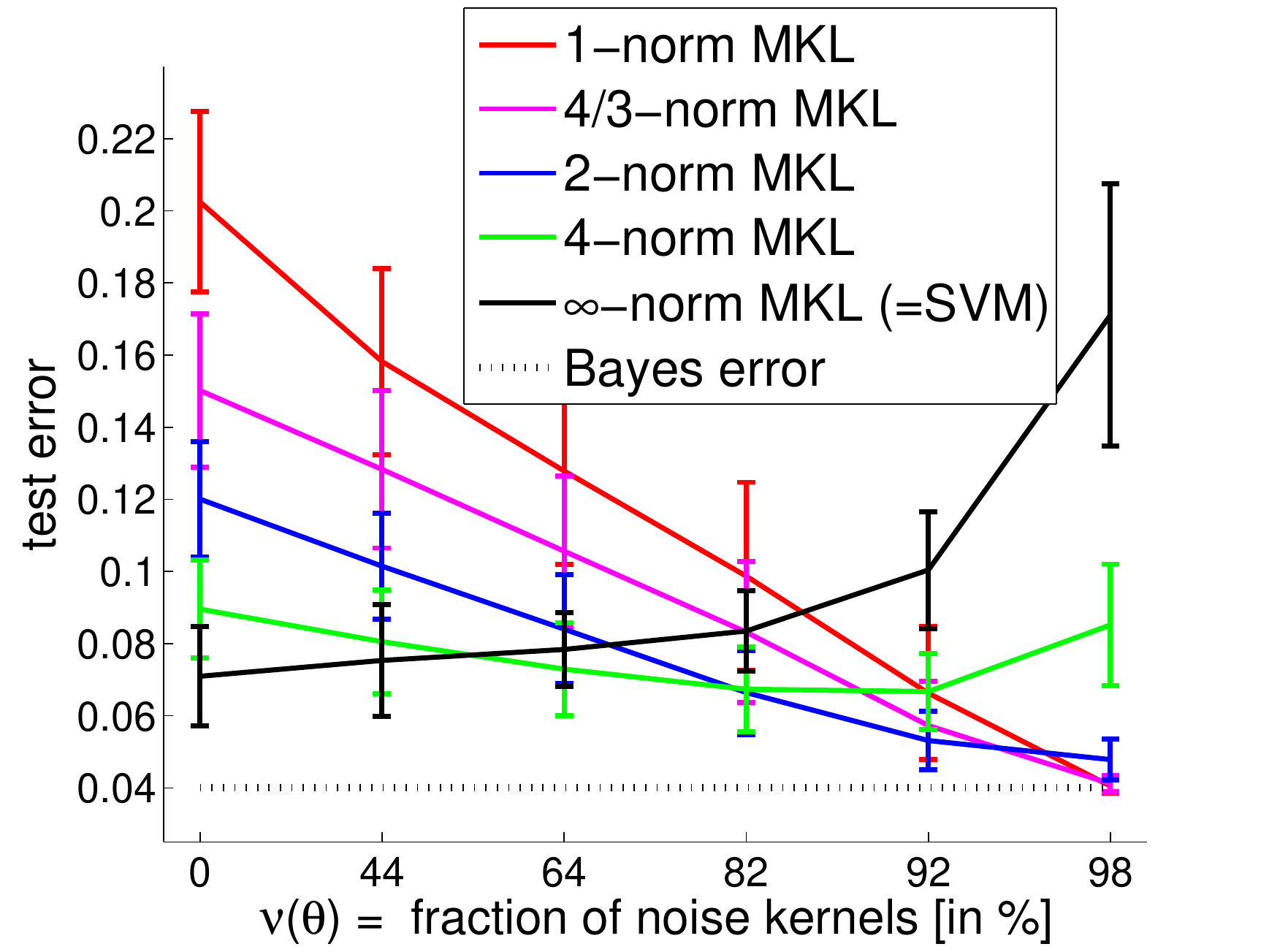}
  }
  \subfigure[]{
    \includegraphics[width=0.5\textwidth]{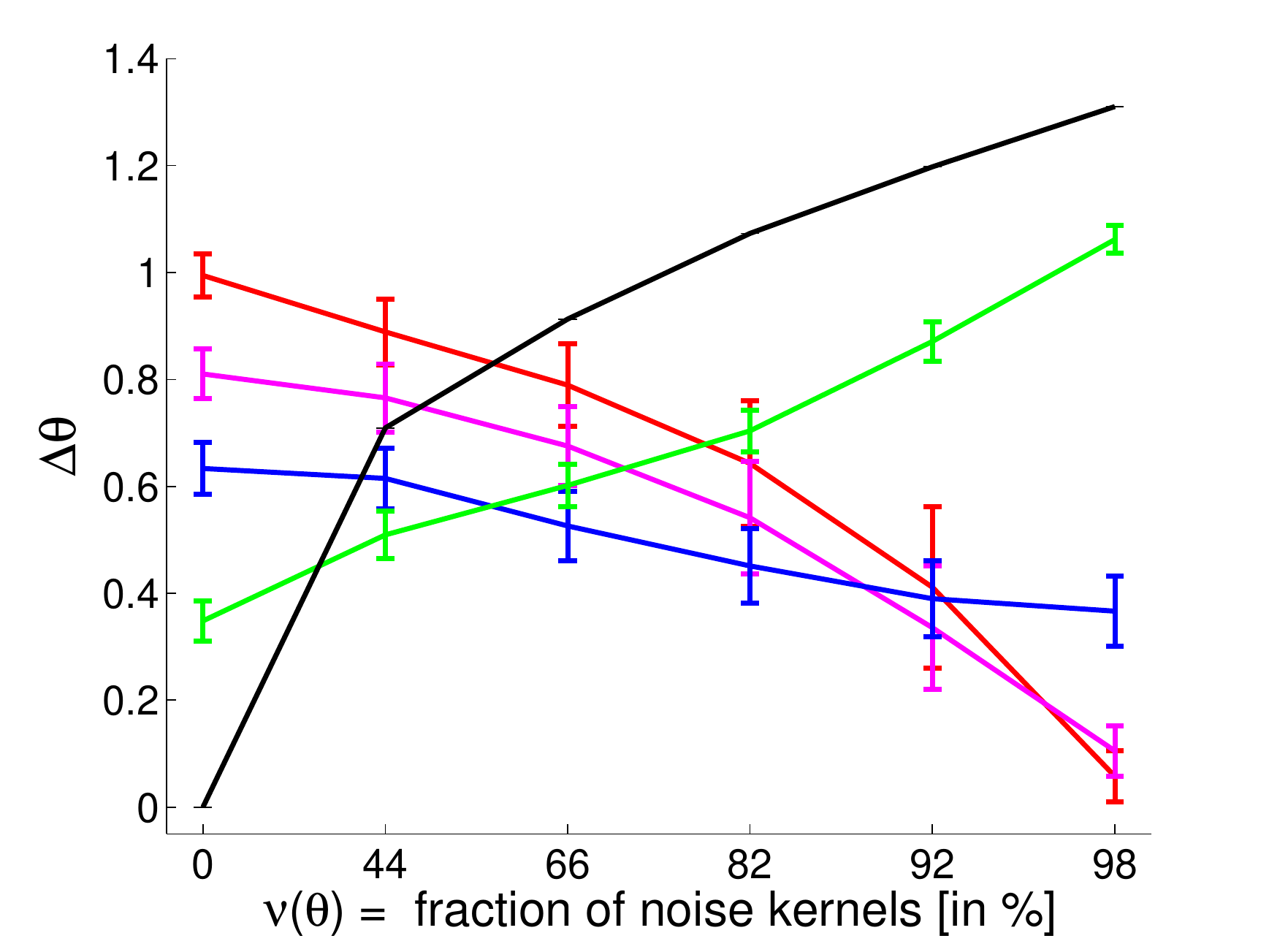}
  }
  \subfigure[]{
    \includegraphics[width=0.49\textwidth]{test_error}
  }
  \subfigure[]{
    \includegraphics[width=0.5\textwidth]{delta_weights}
  }
  \caption{\small{
    Results of the artificial experiment for sample sizes of $n=50$ (top) and $n=800$ (below) training instances in terms of test errors (left) and mean $\ell_2$ model errors ${\rm ME}(\mklvtheta)$ (right).}}
  \label{toy-gap50}
\end{figure}

\blue{Unsurprisingly, $\ell_1$ performs best and reaches the Bayes error in 
the sparse scenario, where only a single kernel carries the whole discriminative information of the learning problem.
However, in the other scenarios it mostly performs worse than the other MKL variants.
This is remarkable because the underlying ground truth, i.e.\ the vector $\truevtheta$, is sparse in all but the uniform scenario.
In other words, selecting this data set may imply a bias towards $\ell_1$-norm.}
\blue{In contrast, the vanilla SVM using an unweighted sum kernel performs best when all kernels are equally informative, 
however, its performance does not approach the Bayes error rate.
This is because it corresponds to a $\ell_{2,2}$-block norm regularization (see Sect.~\ref{section-limits}) but for a truly uniform regularization
a $\ell_\infty$-block norm penalty \cite[as employed in][]{Nathetal09} would be needed.
This indicates a limitation of our framework; it shall, however, be kept in mind that such a uniform scenario might quite artificial.}
The non-sparse $\ell_4$- and $\ell_2$-norm MKL variants perform best in the balanced scenarios, i.e., when the noise level is ranging in the interval 64\%-92\%. Intuitively, the non-sparse $\ell_4$-norm 
MKL is the most robust MKL variant, achieving a test error of less than $10\%$ in all scenarios. 
Tuning the sparsity parameter $p$ for each experiment,   
$\ell_p$-norm MKL achieves the  lowest test error across all scenarios.

When the sample size is increased to $n=800$ training instances, test errors decrease significantly. 
Nevertheless, we still observe differences of up to 1\% test error between the best ($\ell_\infty$-norm MKL) and worst ($\ell_1$-norm MKL) prediction model 
in the two most non-sparse scenarios. Note that all $\ell_{p}$-norm MKL variants perform well in the sparse scenarios.
In contrast with the test errors, the mean model errors depicted in Figure
\ref{toy-gap50} (bottom, right) are relatively high. 
Similarly to above reasoning, this discrepancy can be explained by the minimizer of the empirical risk becoming stable when increasing the sample size (see theoretical Analysis in Appendix~\ref{sect:bounds}, where we show that speed of the minimizer becoming stable is $O(1/\sqrt{n})$).  
Again, $\ell_p$-norm MKL  achieves the smallest test error for all 
scenarios for appropriately chosen $p$ and for a fixed $p$ across all
experiments, the non-sparse $\ell_4$-norm MKL performs the most robustly. 

In summary, the choice of the norm parameter $p$ is important for small 
sample sizes, whereas its impact decreases with an increase of 
the training data. 
As expected, sparse MKL performs best in sparse scenarios, 
while non-sparse MKL performs best in moderate or non-sparse scenarios, and
for uniform scenarios the unweighted-sum kernel SVM performs best.
For appropriately tuning the norm parameter, $\ell_p$-norm MKL proves 
robust in all scenarios.

\subsection{Protein Subcellular Localization---a Sparse Scenario}
\label{sec:prot}
The prediction of the subcellular localization of proteins is one of
the rare empirical success stories of $\ell_1$-norm-regularized MKL
\citep{OngZien08,ZieOng07}: after defining 69 kernels that capture
diverse aspects of protein sequences, $\ell_1$-norm-MKL could raise
the predictive accuracy significantly above that of the unweighted sum
of kernels, and thereby also improve on established prediction systems
for this problem.  This has been demonstrated on 4 data sets,
corresponding to 4 different sets of organisms (plants, non-plant
eukaryotes, Gram-positive and Gram-negative bacteria) with differing
sets of relevant localizations.  In this section, we investigate the
performance of non-sparse MKL on the same 4 data sets.

We downloaded the kernel matrices of all 4 data sets\footnote{Available
 from \url{http://www.fml.tuebingen.mpg.de/raetsch/suppl/protsubloc/}}.
The kernel matrices are multiplicatively normalized as described
in Section \ref{sec:kernelnormalization}.
The experimental setup used here is related to that of
\cite{OngZien08}, although it deviates from it in several details.
For each data set, we perform the following steps for each of the 30
predefined splits in training set and test set (downloaded from the
same URL): We consider norms $p \in \{ 1, 32/31, 16/15, 8/7, 4/3, 2, 4, 8, \infty \}$ and regularization constants $C \in \{1/32,1/8,1/2,1,2,4,8,32,128\}$.
For each parameter setting  $(p,C)$, we train $\ell_p$-norm MKL using a 1-vs-rest 
strategy on the training set. 
The predictions on the test set are then evaluated
w.r.t.\ average (over the classes) MCC (Matthews correlation coefficient).
\blue{As we are only interested in the influence of the norm on the
performance, we forbear proper cross-validation (the so-obtained systematical error affects all norms equally).
Instead, for each of the 30 data splits and for each $p$,
the value of $C$ that yields the highest MCC is selected.
Thus we obtain an optimized $C$ and $MCC$ value for each
combination of data set, split, and norm $p$.
For each norm, the final $MCC$ value is obtained by
averaging over the data sets and splits (i.e., $C$ is selected to be optimal for each data set and split).
}

The results, shown in Table \ref{tab-psl}, indicate that indeed, with
proper choice of a non-sparse regularizer, the accuracy of
$\ell_1$-norm can be recovered.  
On the other hand, non-sparse MKL can approximate the
$\ell_1$-norm arbitrarily close, and thereby approach the same
results.  However, even when $1$-norm is clearly superior to
$\infty$-norm, as for these 4 data sets, it is possible that
intermediate norms perform even better.  As the table shows, this is
indeed the case for the PSORT data sets, albeit only slightly and not
significantly so.

\begin{table}[t]
\begin{center}
\caption{ \label{tab-psl} \small{
  Results for Protein Subcellular Localization.
  For each of the 4 data sets (rows) and each considered norm (columns),
  we present a measure of prediction error together with its
  standard error.  As measure of prediction error we use 1 minus
  the average MCC, displayed as percentage.}
}
\medskip
\begin{tabular}{l|r|r|r|r|r|r|r|r|r|r}
{\bf $\ell_p$-norm} & $1$ & $32/31$ & $16/15$ & $8/7$ & $4/3$ & $2$ & $4$ & $8$ & $16$ & $\infty$ \\
\hline
\hline \mbox{{\bf plant}} & {\bf 8.18} & 8.22 & 8.20 & 8.21 & 8.43 & 9.47 & 11.00 & 11.61 & 11.91 & 11.85 \\
\quad std.~err. 
& $\pm$0.47 & $\pm$0.45 & $\pm$0.43 & $\pm$0.42 & $\pm$0.42 & $\pm$0.43 & $\pm$0.47 & $\pm$0.49 & $\pm$0.55 & $\pm$0.60 \\
\hline {\bf nonpl} & {\bf 8.97} & 9.01 & 9.08 & 9.19 & 9.24 & 9.43 & 9.77 & 10.05 & 10.23 & 10.33 \\
\quad std.~err. 
& $\pm$0.26 & $\pm$0.25 & $\pm$0.26 & $\pm$0.27 & $\pm$0.29 & $\pm$0.32 & $\pm$0.32 & $\pm$0.32 & $\pm$0.32 & $\pm$0.31 \\
\hline {\bf psortNeg} & 9.99 & 9.91 & {\bf 9.87} & 10.01 & 10.13 & 11.01 & 12.20 & 12.73 & 13.04 & 13.33 \\
\quad std.~err. 
& $\pm$0.35 & $\pm$0.34 & $\pm$0.34 & $\pm$0.34 & $\pm$0.33 & $\pm$0.32 & $\pm$0.32 & $\pm$0.34 & $\pm$0.33 & $\pm$0.35 \\
\hline {\bf psortPos} & 13.07 & {\bf 13.01} & 13.41 & 13.17 & 13.25 & 14.68 & 15.55 & 16.43 & 17.36 & 17.63 \\
\quad std.~err. 
& $\pm$0.66 & $\pm$0.63 & $\pm$0.67 & $\pm$0.62 & $\pm$0.61 & $\pm$0.67 & $\pm$0.72 & $\pm$0.81 & $\pm$0.83 & $\pm$0.80 \\
\end{tabular}
\end{center}
\end{table}



We briefly mention that the superior performance of $\ell_{p\approx
1}$-norm MKL in this setup is not surprising. There are four sets of
16 kernels each, in which each kernel picks up very similar
information: they only differ in number and placing of gaps in all
substrings of length 5 of a given part of the protein sequence. The
situation is roughly analogous to considering (inhomogeneous)
polynomial kernels of different degrees on the same data vectors.
This means that they carry large parts of overlapping information.  By
construction, also some kernels (those with less gaps) in principle
have access to more information (similar to higher degree polynomials
including low degree polynomials).  Further, \cite{OngZien08} studied
single kernel SVMs for each kernel individually and found that in most
cases the 16 kernels from the same subset perform very similarly. 
This means that each set of 16 kernels is highly redundant and
the excluded parts of information are not very discriminative.
This renders a non-sparse kernel mixture ineffective. We conclude that
$\ell_1$-norm must be the best prediction model.

\subsection{Gene Start Recognition---a Weighted Non-Sparse Scenario}\label{sec:tss}

This experiment aims 
at detecting transcription start sites (TSS) of
RNA Polymerase II binding genes in genomic DNA sequences.
Accurate detection of the
transcription start site is crucial to identify genes and their
promoter regions and can be regarded as a first step in
deciphering the key regulatory elements in the promoter region that
determine transcription.

Transcription start site finders exploit the fact that the features
of promoter regions and the transcription start sites 
are different from the features of other  genomic DNA \citep{BajTanSuzSug04}.
Many such detectors thereby rely on a combination of 
feature sets which makes the learning task appealing for MKL. 
For our experiments we use the data set from \cite{SonZieRae06} 
which contains a curated set of 8,508 TSS annotated genes utilizing
dbTSS version 4 \citep{SuzYamNakSug02} and refseq genes.
These are translated into positive training instances 
by extracting windows of size $[-1000,+1000]$ around the TSS. 
Similar to \cite{BajTanSuzSug04}, 85,042 negative instances are 
generated from the interior of the gene using the same window size. 

Following \cite{SonZieRae06}, we employ five different kernels representing the TSS signal 
(weighted degree with shift), the  promoter (spectrum), 
the 1st exon (spectrum), angles (linear), 
and energies (linear). Optimal kernel parameters are determined
by model selection in \cite{SonZieRae06}.
The kernel matrices are spherically normalized as described
in section \ref{sec:kernelnormalization}.
We reserve 13,000 and 20,000 randomly drawn instances for validation and 
test sets, respectively, and use the remaining 60,000 as the training pool.
Soft  margin parameters $C$ are tuned on the validation set by grid 
search over $C\in 2^{[-2,-1,\ldots,5]}$ (optimal $C$s are attained in the interior of the grid). 
Figure~\ref{fig:tss} shows test errors for varying training set sizes
drawn from the pool; training sets of the same size are disjoint.
Error bars indicate standard errors of 
repetitions for small training set sizes. 

Regardless of the sample size, $\ell_1$-norm MKL is significantly outperformed
by the sum-kernel. On the contrary, non-sparse MKL significantly
achieves higher AUC values than the $\ell_\infty$-norm MKL for sample sizes
up to 20k. The scenario is well suited for $\ell_2$-norm MKL which
performs best. Finally, for 60k training instances, all methods but
$\ell_1$-norm MKL yield the same performance. 
Again, the superior performance of non-sparse MKL is remarkable, and
of significance for the application domain: the method using the
unweighted sum of kernels \citep{SonZieRae06} has recently been
confirmed to be leading in a comparison of 19 state-of-the-art
promoter prediction programs \citep{AbeelPS09}, and our experiments
suggest that its accuracy can be further elevated by non-sparse MKL.

\begin{figure}[t]
  \centering
  \hspace{-0.6cm}
  \includegraphics[width=0.56\textwidth]{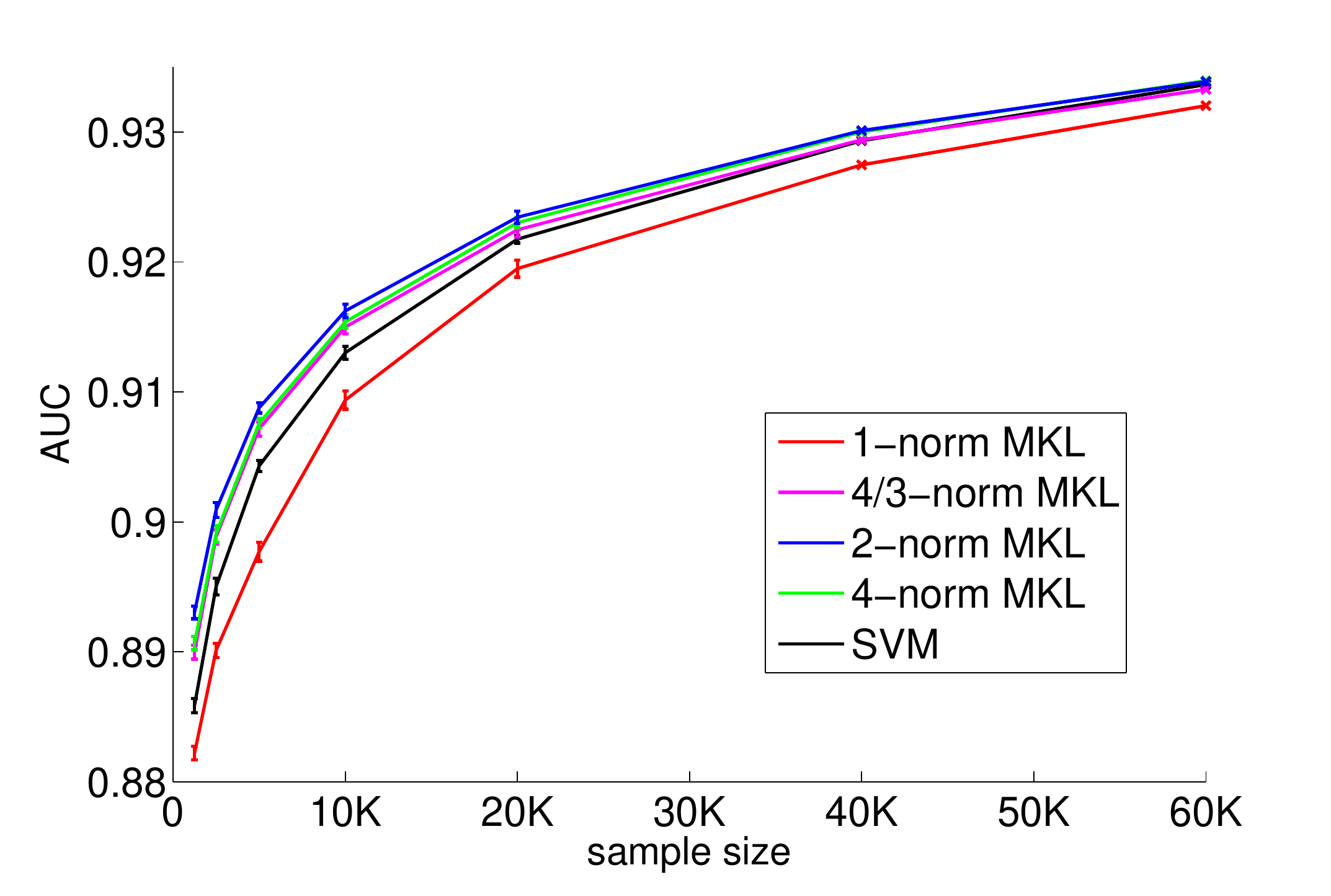}
  \hspace{-0.75cm}
  \includegraphics[width=0.5\textwidth]{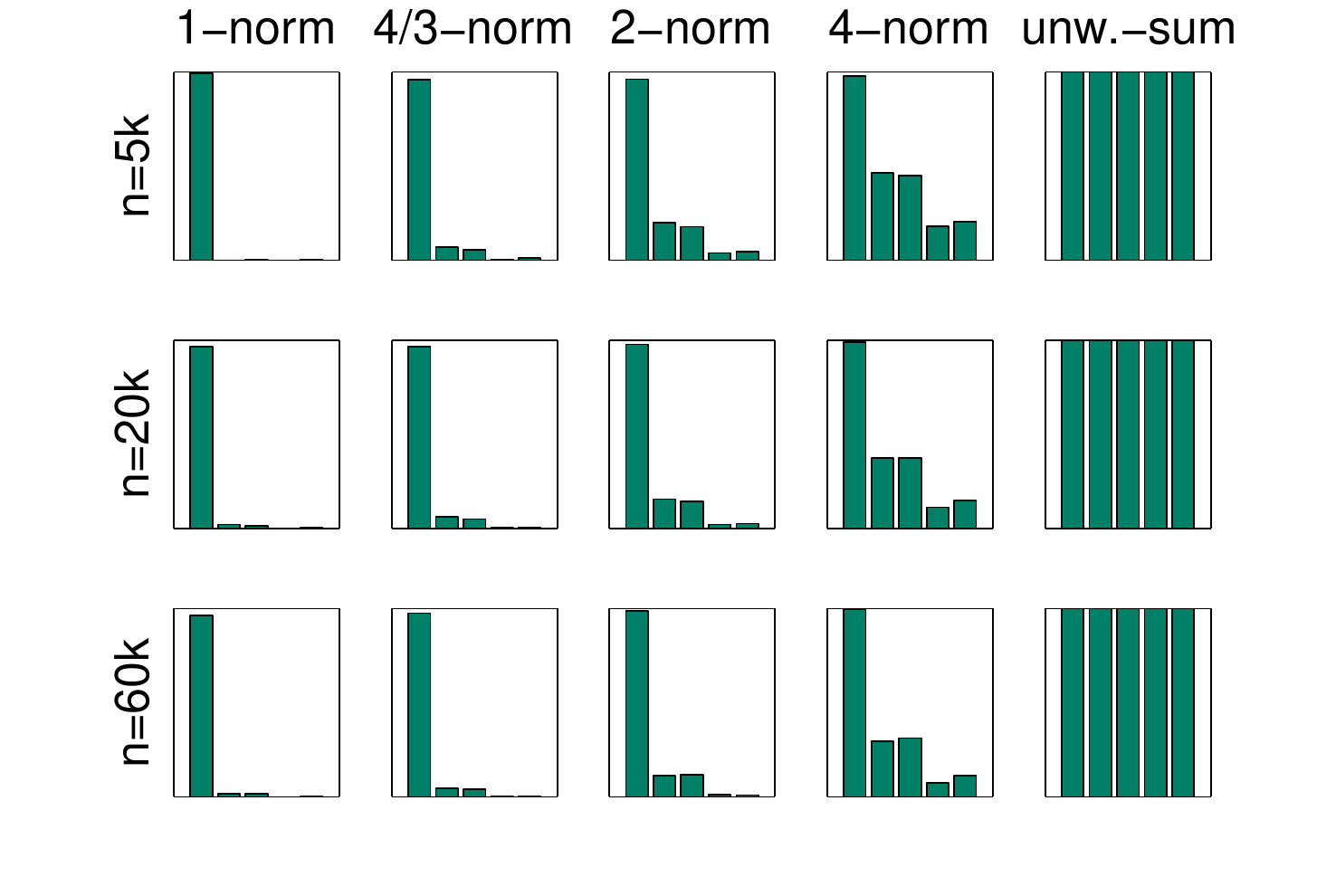}
  \caption{ \small \label{fig:tss}{
    (left) Area under ROC curve (AUC) on test data for TSS recognition
    as a function of the training set size.  Notice the tiny bars indicating standard errors
    w.r.t.~repetitions on disjoint training sets. (right) Corresponding kernel mixtures. For $p=1$
    consistent sparse solutions are obtained while the optimal $p=2$ distributes weights on the 
    weighted degree and the 2 spectrum kernels in good agreement to \citep{SonZieRae06}.
}}
\end{figure}


We give a brief explanation of the reason for optimality of a
non-sparse $\ell_p$-norm in the above experiments.  It has been shown by
\cite{SonZieRae06} that there are three highly and two moderately
informative kernels.  We briefly recall those results by reporting on
the AUC performances obtained from training a single-kernel SVM on
each kernel individually: TSS signal $0.89$, promoter $0.86$, 1st exon
$0.84$, angles $0.55$, and energies $0.74$, for fixed sample size
$n=2000$.  While non-sparse MKL distributes the weights over all
kernels (see Fig.~\ref{fig:tss}), sparse MKL focuses on the best
kernel.  However, the superior performance of non-sparse MKL means
that dropping the remaining kernels is detrimental, indicating
that they may carry additional discriminative information.

To investigate this hypothesis we computed the pairwise alignments\footnote{The
alignments can be interpreted as empirical estimates of the Pearson
correlation of the kernels \citep{Cristianini2002}.} of the centered
kernel matrices, i.e., $\mathcal A(i,j) = \frac{<K_i,K_j>_F}{\Vert K_i\Vert_F
\Vert K_j \Vert_F}$, with respect to the Frobenius dot product
\cite[e.g.,][]{GolLoa96}.  The computed alignments are shown in
Fig.~\ref{fig:tss_alignments}.  One can observe that the three
relevant kernels are highly aligned as expected since they are
correlated via the labels.

However, the energy kernel shows only a slight correlation with the
remaining kernels, which is surprisingly little compared to its single
kernel performance (AUC=$0.74$).  We conclude that this kernel
carries complementary and orthogonal information about the learning
problem and should thus be included in the resulting kernel
mixture. This is precisely what is done by non-sparse MKL, as can be
seen in Fig.~\ref{fig:tss}(right), and the reason for the empirical
success of non-sparse MKL on this data set.

\begin{figure}[t]
  \centering
  \includegraphics[width=0.50\textwidth]{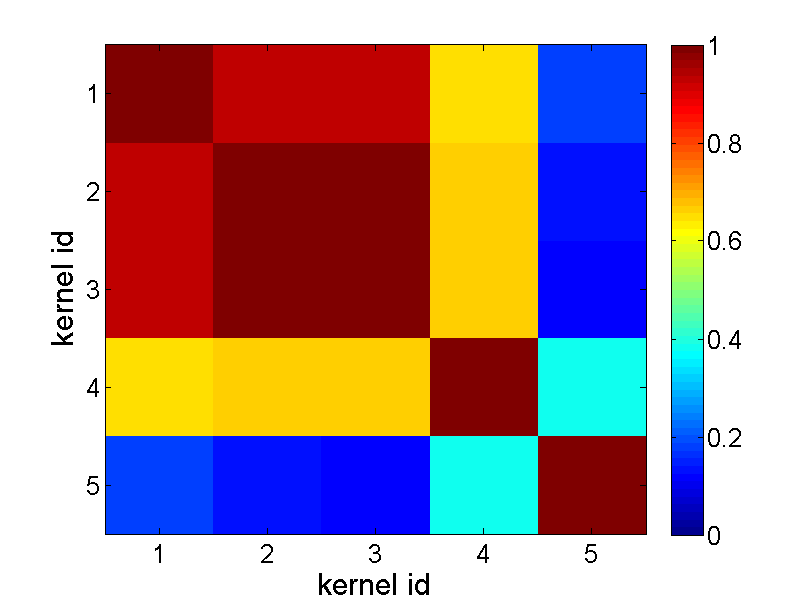}
  \caption{ \small  \label{fig:tss_alignments}{Pairwise alignments of the kernel matrices are shown for the gene start recognition experiment. From left to right, the ordering of the kernel matrices is TSS signal, 
    promoter, 1st exon, angles, and energies. The first three kernels are highly correlated, as expected by their high AUC performances (AUC=$0.84$--$0.89$)
    and the angle kernel correlates decently (AUC=$0.55$). Surprisingly,  the energy kernel correlates only few, despite a descent AUC of $0.74$.}}
\end{figure}

\subsection{Reconstruction of Metabolic Gene Network---a Uniformly Non-Sparse Scenario}\label{sec:bleakley}

In this section, we apply non-sparse MKL to a problem originally studied
by \cite{YamanishiEtAl2005}. Given 668 enzymes of the yeast 
\emph{Saccharomyces cerevisiae} and 2782 functional relationships 
extracted from the KEGG database \citep{KanehisaEtAl2004}, the task is to 
predict functional relationships for unknown enzymes. We employ 
the experimental setup of \cite{BleakleyEtAl2007} who phrase the 
task as graph-based edge prediction with local models by learning
a model for each of the 668 enzymes. They provided 
kernel matrices capturing expression data (EXP), cellular 
localization (LOC), and the phylogenetic profile (PHY); additionally 
we use the integration of the former 3 kernels (INT) which matches our 
definition of an unweighted-sum kernel. 

Following \cite{BleakleyEtAl2007}, we employ a $5$-fold cross validation;
in each fold we train on average 534 enzyme-based models; however, 
in contrast to  \cite{BleakleyEtAl2007} we omit enzymes reacting with 
only one or two others to guarantee well-defined problem settings.
As Table \ref{diabolisch2000} shows, this results in slightly better 
AUC values for single kernel SVMs where the results by \cite{BleakleyEtAl2007}
are shown in brackets. 

As already observed \citep{BleakleyEtAl2007}, the unweighted-sum
kernel SVM performs best. Although its solution is well approximated
by non-sparse MKL using large values of $p$, $\ell_p$-norm MKL is not
able to improve on this $p=\infty$ result. Increasing the number of
kernels by including recombined and product kernels does improve the
results obtained by MKL for small values of $p$, but the maximal AUC
values are not statistically significantly different from those of
$\ell_\infty$-norm MKL.
We conjecture that the performance of the unweighted-sum kernel SVM
can be explained by all three kernels performing well invidually. 
Their correlation is only  moderate, as shown in Fig.~\ref{fig:bleakley_alignments},
suggesting that they contain complementary information. 
Hence, downweighting one of those three orthogonal 
kernels leads to a decrease in performance, as observed in our experiments.
This explains why $\ell_\infty$-norm MKL is the best prediction model in this experiment. 

\begin{table}
\begin{center}
\caption{\small{Results for the reconstruction of a metabolic gene network.
\label{diabolisch2000} Results by \cite{BleakleyEtAl2007} for single 
kernel SVMs are shown in brackets.}}
\medskip
\begin{tabular}{l|l}
            & AUC $\pm$ stderr     \\\hline\hline
EXP         &   $71.69 \pm 1.1$ $\quad$ ($69.3 \pm 1.9$)\\
LOC         &   $58.35 \pm 0.7$ $\quad$ ($56.0 \pm 3.3$)\\
PHY         &   $73.35 \pm 1.9$ $\quad$ ($67.8 \pm 2.1$)\\
INT ($\infty$-norm MKL)  &   $\mathbf{82.94} \pm \mathbf{1.1}$ $\,\,\,$ ($\mathbf{82.1} \pm \mathbf{2.2}$)\\\hline\hline
$1$-norm MKL & $75.08 \pm 1.4$ \\
${4/3}$-norm MKL & $  78.14 \pm 1.6$\\
$2$-norm MKL     &  $ 80.12 \pm 1.8$\\
$4$-norm    MKL   &  $ 81.58 \pm 1.9$\\
$8$-norm   MKL    &  $ 81.99 \pm 2.0$\\
$10$-norm  MKL    &  $ \mathbf{82.02} \pm \mathbf{2.0}$\\\hline
\multicolumn{2}{l}{Recombined and product kernels}\\\hline
$1$-norm MKL    &  $79.05 \pm 0.5$ \\
$4/3$-norm MKL  & $80.92 \pm 0.6$  \\
$2$-norm MKL    & $81.95 \pm 0.6$  \\
$4$-norm MKL    & $\mathbf{83.13} \pm \mathbf{0.6}$  \\
\end{tabular}
\end{center}
\end{table}

\begin{figure}[t]
  \centering
  \vskip2cm
  \includegraphics[width=0.50\textwidth]{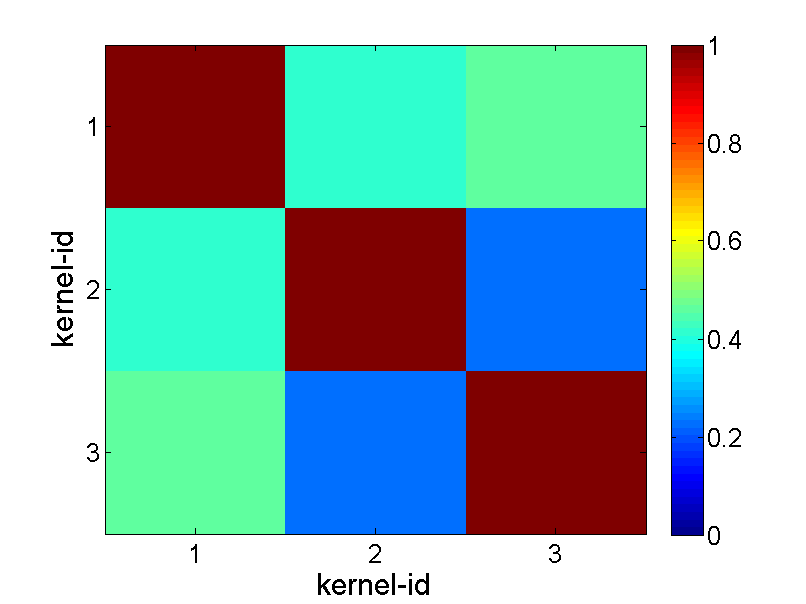}
  \caption{ \small \label{fig:bleakley_alignments}{Pairwise alignments of the kernel matrices are shown for the metabolic gene network experiment. From left to right, the ordering of the kernel matrices is EXP, LOC, and PHY. One can see that all kernel matrices are equally correlated. Generally, the alignments are relatively low, suggesting that combining all kernels with equal weights is beneficial.}}
\end{figure}

\subsection{Execution Time}\label{sec-exectime}

In this section we demonstrate the efficiency of our implementations
of non-sparse MKL. We experiment on the MNIST data set\footnote{This data set is
available from \url{http://yann.lecun.com/exdb/mnist/}.},
where the task
is to separate odd vs.\ even digits. The digits in this $n=60,000$-elemental data set are of size
28x28 leading to $d=784$ dimensional examples.
We compare our analytical solver  for  non-sparse MKL (Section~\ref{opt_alex}--\ref{analytical}) with the state-of-the art  for $\ell_1$-norm MKL, namely SimpleMKL\footnote{We obtained an implementation from \url{http://asi.insa-rouen.fr/enseignants/~arakotom/code/}.}
\citep{RakBacCanGra08}, HessianMKL\footnote{We obtained an implementation from \url{http://olivier.chapelle.cc/ams/hessmkl.tgz}.} \citep{ChaRak08}, SILP-based wrapper, and SILP-based chunking optimization \citep{SonRaeSchSch06}. 
We also experiment with the analytical method for $p=1$, although convergence is only guaranteed by our Theorem~\ref{thm:directmkl} for $p>1$.
We also compare to the semi-infinite program (SIP) approach to $\ell_p$-norm MKL presented in \cite{KloBreSonZieLasMue09}.
\footnote{The Newton method presented in the same paper performed similarly most of the time but sometimes had convergence problems, especially when $p\approx 1$ and thus was excluded from the presentation.}
In addition, we solve standard SVMs\footnote{We use SVMlight as SVM-solver.}
using the unweighted-sum kernel ($\ell_\infty$-norm MKL) as baseline.  

We experiment with MKL using precomputed kernels
(excluding the kernel computation time from the timings) and MKL based on
on-the-fly computed kernel matrices measuring training time \emph{including
kernel computations}. Naturally, runtimes of  on-the-fly methods should be expected to be higher than the
ones of the precomputed counterparts. We optimize all methods up to a precision of $10^{-3}$ for the outer
SVM-$\varepsilon$ and $10^{-5}$ for the ``inner'' SIP precision, and computed
relative duality gaps. To provide a fair stopping criterion to SimpleMKL and HessianMKL, we set
their stopping criteria to the relative duality gap of their 
$\ell_1$-norm SILP counterpart.  
SVM trade-off parameters are set to $C=1$ for all methods.

\begin{figure}[htbp]
\centering
\includegraphics[width=.97\textwidth]{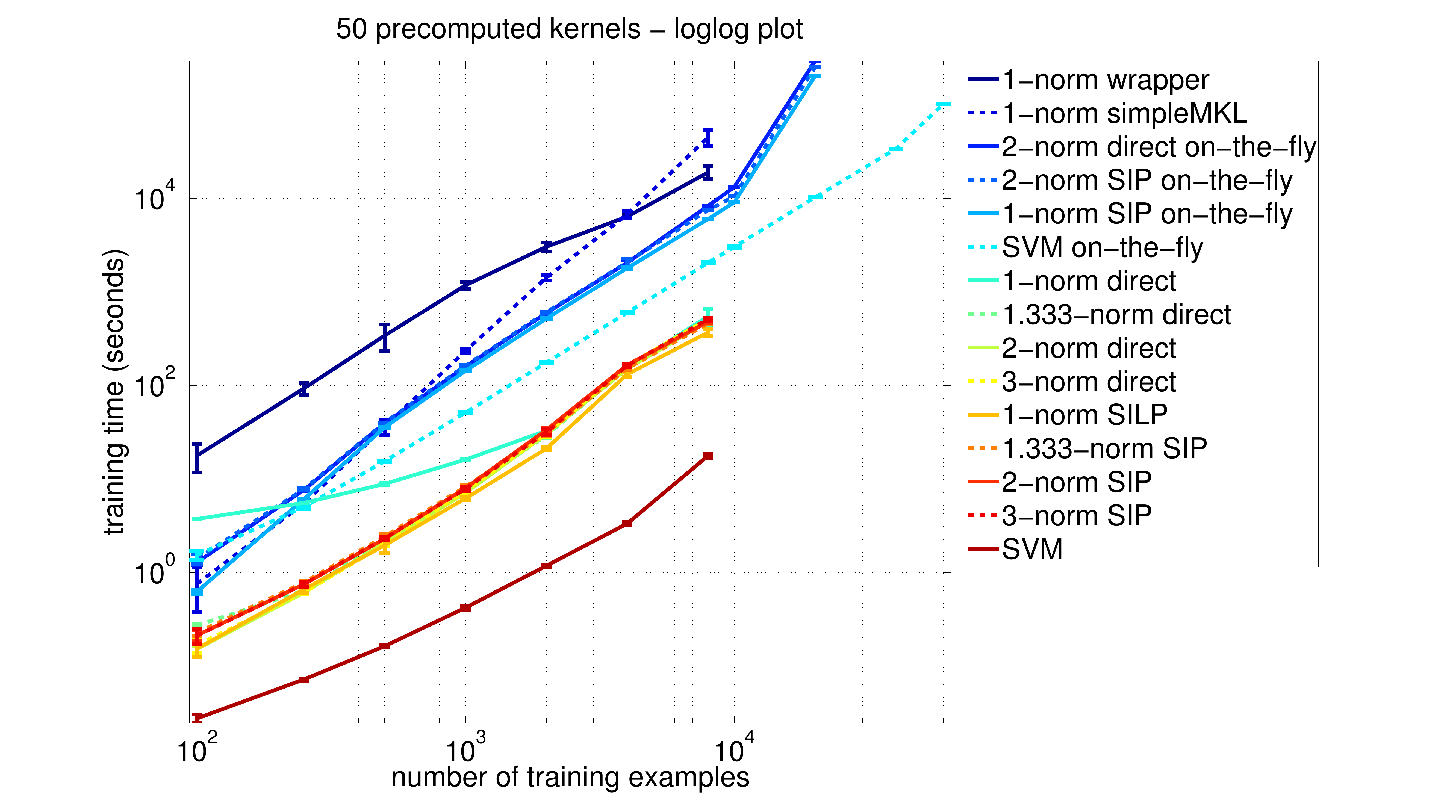}\vspace*{1ex}
\includegraphics[width=.96\textwidth]{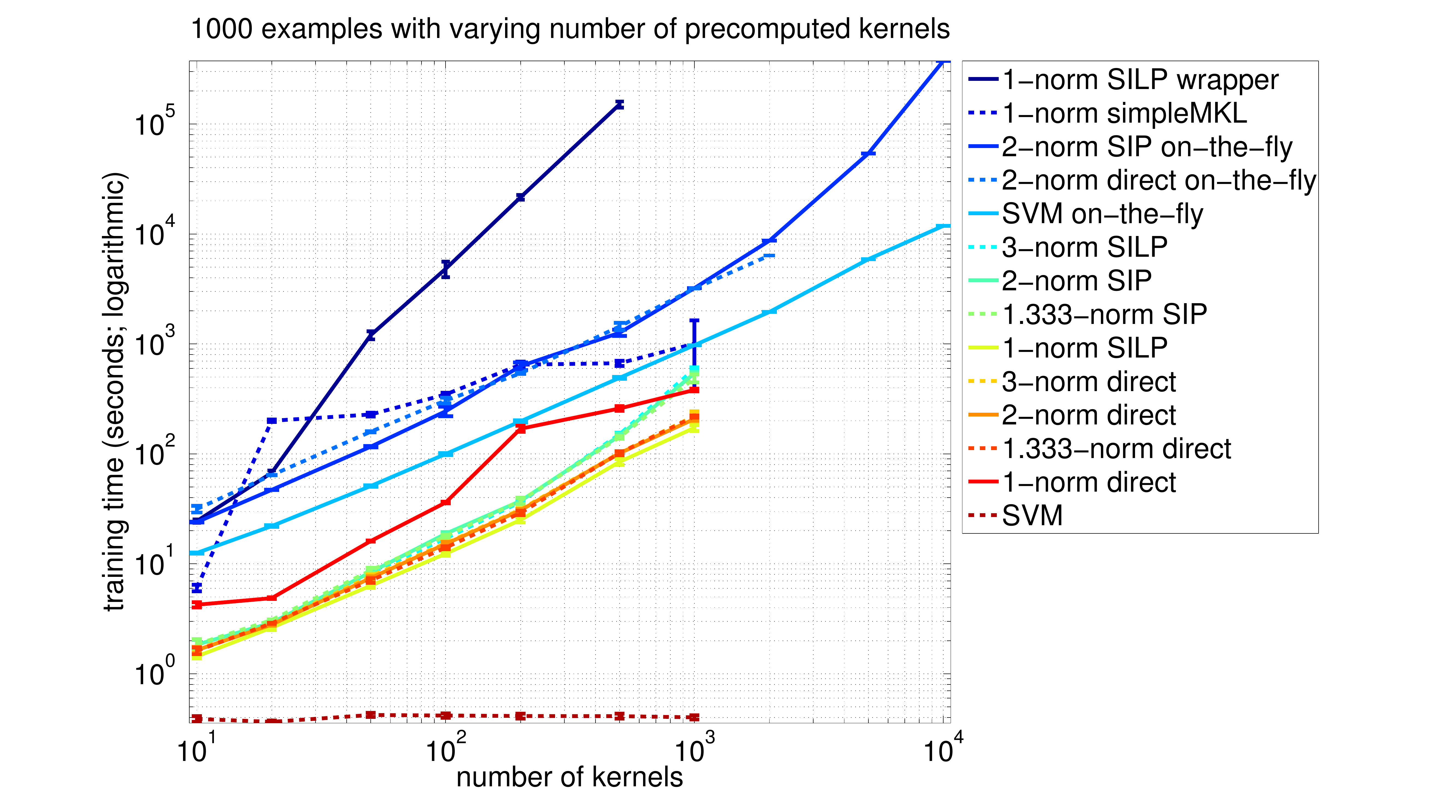}\hspace{0.25cm}
\caption{{\small \blue{Execution times of SVM and $\ell_p$-norm MKL based on interleaved
optimization via analytical optimization and semi-infinite programming (SIP), respectively, and wrapper-based optimization via SimpleMKL
wrapper and SIP wrapper.} Top: Training using fixed number of 50 kernels varying training set
size. Bottom: For 1000 examples and varying numbers of kernels. Notice the tiny
error bars and that these are log-log plots. \label{fig-exectime}}}
\end{figure}

\paragraph{Scalability of the Algorithms w.r.t.~Sample Size}
Figure \ref{fig-exectime} (top) displays the results for varying sample sizes
and 50 precomputed or on-the-fly computed Gaussian kernels with bandwidths $2\sigma^2 \in
1.2^{0,\dots,49}$. Error bars indicate standard error over 5 repetitions.
As expected, the SVM with the unweighted-sum kernel using precomputed kernel
matrices is the fastest method. The classical MKL wrapper based methods, SimpleMKL and
the SILP wrapper, are the slowest; they are even slower than methods that compute
kernels on-the-fly. Note that the on-the-fly methods naturally have higher runtimes because they
do not profit from precomputed kernel matrices.

Notably, when considering 50 kernel matrices of size 8,000 times 8,000 (memory
requirements about 24GB for double precision numbers), SimpleMKL is the slowest method:
it is more than 120 times slower than the $\ell_1$-norm SILP solver from
\cite{SonRaeSchSch06}. This  is because  SimpleMKL suffers from having to train an SVM to full precision
for each gradient evaluation. In contrast, kernel caching and interleaved optimization
still allow to train our algorithm on kernel matrices of size $20000\times20000$, which would
usually not completely fit into memory since they require about 149GB.

Non-sparse MKL scales similarly as
$\ell_1$-norm SILP for both optimization strategies, the analytic optimization and the sequence of SIPs. 
Naturally, the generalized SIPs are slightly slower than the SILP variant, since they solve an additional series of Taylor expansions
within each $\vtheta$-step. \blue{HessianMKL ranks in between on-the-fly and non-sparse interleaved methods.}

\paragraph{Scalability of the Algorithms w.r.t.~the Number of Kernels}
Figure \ref{fig-exectime} (bottom) shows the results for varying the
number of precomputed and on-the-fly computed RBF kernels for a fixed sample
size of 1000. The bandwidths of the kernels are scaled such that for $M$ kernels
$2\sigma^2\in 1.2^{0,\dots,M-1}$. As expected, the
SVM with the unweighted-sum kernel is hardly affected by this setup,
taking an essentially constant training time. The $\ell_1$-norm MKL by
\cite{SonRaeSchSch06} handles the increasing number of kernels best
and is the fastest MKL method. Non-sparse approaches to MKL show
reasonable run-times, being just slightly slower. Thereby the analytical methods are somewhat faster than the SIP approaches.
\blue{The sparse analytical method performs worse than its non-sparse counterpart; 
this might be related to the fact that convergence of the analytical method is only guaranteed for $p>1$.}
The wrapper methods again perform
worst. 

However, in contrast to the previous experiment, SimpleMKL
becomes more efficient with increasing number of kernels.
We conjecture that this is in part owed to the sparsity
of the best solution, which accommodates the $l_1$-norm model of SimpleMKL.
But the capacity of SimpleMKL remains limited due to memory
restrictions of the hardware. For example, for storing 1,000 kernel
matrices for 1,000 data points, about 7.4GB of memory are required.
On the other hand, our interleaved optimizers which allow for effective caching
can easily cope with 10,000 kernels of the same size (74GB). 
\blue{HessianMKL is considerably faster than SimpleMKL but slower than the non-sparse interleaved methods and the SILP.
Similar to SimpleMKL, it becomes more efficient with increasing number of kernels but eventually runs out of memory.}

\blue{Overall, our proposed interleaved analytic and cutting plane based optimization strategies
achieve a speedup of up to one and two orders of magnitude over HessianMKL and SimpleMKL, respectively.} Using
efficient kernel caching, they allow for truely large-scale multiple kernel learning
well beyond the limits imposed by having to precompute and store the complete kernel matrices.
Finally, we note that performing MKL with 1,000 precomputed kernel matrices of size 1,000 times 1,000 requires
less than 3 minutes for the SILP. This suggests that it focussing future research efforts
on improving the accuracy of MKL models may pay off more than further accelerating the
optimization algorithm.


\section{Conclusion}\label{SEC-conclusion}

We translated multiple kernel learning into a regularized risk 
minimization problem for arbitrary convex loss functions,
Hilbertian regularizers, and arbitrary-norm penalties on the
mixing coefficients. Our formulation can be motivated
by both Tikhonov and Ivanov regularization approaches, the latter one
having an additional regularization parameter.
Applied to previous MKL research, our framework provides a unifying view
and shows that so far seemingly different MKL approaches are in fact equivalent.

Furthermore, we presented a general dual formulation of multiple 
kernel learning that subsumes many existing algorithms. We devised 
an efficient optimization scheme for non-sparse $\ell_p$-norm 
MKL with $p\geq1$, based on an analytic update for the mixing coefficients, and 
interleaved with chunking-based SVM training to 
allow for application at large scales.
It is an open question whether our algorithmic approach extends to more general norms.
Our implementations are freely available and
included in the SHOGUN toolbox. 
The execution times of our algorithms 
revealed that the interleaved optimization 
vastly outperforms commonly used wrapper approaches. 
 Our results and the scalability of our MKL approach 
pave the way for other real-world applications of multiple kernel learning.

In order to empirically validate our $\ell_p$-norm MKL model,
we applied it to artificially generated
data and real-world problems from computational biology. 
For the controlled toy experiment, where we simulated various levels
of sparsity, $\ell_p$-norm MKL achieved a low test error in all
scenarios for scenario-wise tuned parameter $p$.
Moreover, we studied three real-world problems showing that 
the choice of the norm is crucial for state-of-the art performance.
For the TSS recognition, non-sparse MKL raised the bar in 
predictive performance, while for the other two tasks either sparse MKL
or the unweighted-sum mixture performed best. In those cases
the best solution can be arbitrarily closely approximated
by $\ell_p$-norm MKL with $1<p<\infty$.
Hence it seems natural that we observed non-sparse MKL to be never worse than
an unweighted-sum kernel or a sparse MKL approach. 
Moreover, empirical evidence from our experiments along with others
suggests that the popular $\ell_1$-norm MKL is more prone to bad solutions
than higher norms, despite appealing guarantees like the model selection consistency \citep{Bac07}.

\blue{A first step towards a learning-theoretical understanding of this empirical behaviour
may be the convergence analysis undertaken in the appendix of this paper.
It is shown that in a sparse scenario $\ell_1$-norm MKL converges faster
than non-sparse MKL due to a bias that well is well-taylored to the ground truth.
In their current form the bounds seem to suggest that furthermore,
in all other cases, $\ell_1$-norm MKL is at least as good as non-sparse MKL.
However this would be inconsistent with both the no-free-lunch theorem
and our empirical results, which indicate that there exist scenarios
in which non-sparse models are advantageous.
We conjecture that the non-sparse bounds are not yet tight and need
further improvement, for which the results in Appendix~\ref{sect:bounds}
may serve as a starting point.}\footnote{\blue{We conjecture that the $\ell_{p>1}$-bounds are off by a logarithmic factor, because our proof technique ($\ell_1$-to-$\ell_p$ conversion) introduces a slight bias towards $\ell_1$-norm.}}

A related---and obtruding!---question is whether the optimality of the parameter $p$
can retrospectively be explained or, more profitably, even be estimated in advance.
Clearly, cross-validation based model selection over the choice of $p$
will inevitably tell us which cases call for sparse or non-sparse models.
The analyses of our real-world applications suggests that both
the correlation amongst the kernels with each other and their correlation
with the target (i.e., the amount of discriminative information that
they carry) play a role in the distinction of sparse from non-sparse scenarios.
However, the exploration of theoretical explanations is beyond the scope of this work.
Nevertheless, we remark that even completely redundant but uncorrelated kernels 
may improve the predictive performance of a model, as averaging
over several of them can reduce the variance of the predictions \cite[cf., e.g.,][Sect. 3.1]{GuyEli03}.
Intuitively speaking, we observe clearly that in some cases all
features, even though they may contain redundant information, should be
kept, since putting their contributions to zero worsens
prediction, i.e.\ all of them are informative to our MKL models. 

Finally, we would like to note that it may be worthwhile to rethink the
current strong preference for sparse models in the scientific community.
%
%
\blue{Already weak connectivity in a causal graphical model may be sufficient
for all variables to be required for optimal predictions (i.e., to have non-zero coefficients),
and even the prevalence of sparsity in causal flows is being questioned
(e.g., for the social sciences \citet{Gelman10} argues that "There are (almost) no true zeros").}
A main reason for favoring sparsity may be the presumed
interpretability of sparse models.  This is not the topic and goal of
this article; however we remark that in general the identified model
is sensitive to kernel normalization, and in particular in the
presence of strongly correlated kernels the results may be somewhat
arbitrary, putting their interpretation in doubt.  However, in the
context of this work the predictive accuracy is of focal interest, and
in this respect we demonstrate that non-sparse models may improve
quite impressively over sparse ones.

\begin{acks}
The authors wish to thank Vojtech Franc, Peter Gehler, Pavel Laskov, Motoaki Kawanabe, and
Gunnar R\"atsch for stimulating discussions, and Chris Hinrichs  and Klaus-Robert M\"uller for helpful comments on
the manuscript. 
We acknowledge Peter L. Bartlett and Ulrich R\"uckert for contributions to parts of an earlier version 
of the theoretical analysis that appeared at ECML 2010.
We thank the anonymous reviewers for comments and suggestions 
that helped to improve the manuscript.
This work was supported in part by the German
Bundesministerium f\"ur Bildung und Forschung (BMBF) under the project REMIND
(FKZ 01-IS07007A), and by the FP7-ICT
program of the European Community, under the PASCAL2  Network of Excellence,
ICT-216886. S\"oren Sonnenburg 
acknowledges financial support by the German Research Foundation
(DFG) under the grant MU 987/6-1 and RA 1894/1-1, and Marius Kloft acknowledges a scholarship by the German Academic Exchange Service (DAAD).
\end{acks}

\appendix

\section{Theoretical Analysis}\label{sect:bounds}

\blue{
In this section we present a theoretical analysis of $\ell_p$-norm MKL, based on Rademacher complexities.\footnote{An excellent introduction to statistical learning theory, which equips the reader with the needed basics for this section, is given in \cite{BouBouLug04}.}
We prove a theorem that converts any Rademacher-based generalization bound on $\ell_1$-norm MKL into a 
generalization bound for $\ell_p$-norm MKL (and even more generally: arbitrary-norm MKL).
Remarkably this $\ell_1$-to-$\ell_p$ conversion is obtained almost without any effort: by a
simple 5-line proof. The proof idea is based on \cite{KloRueBar10}.\footnote{We acknowledge the contribution of Ulrich R\"uckert.}
We remark that an  $\ell_p$-norm MKL bound  was already given in \cite{CorMohRos10}, 
but first their bound is only valid for the special cases
$p = n/(n-1)$ for $n=1,2,\ldots$,
and second it is not tight for all $p$, as it diverges to infinity when $p>1$ and $p$ approaches one.  By contrast, beside a rather unsubstantial $\log(M)$-factor, our result matches the best known lower bounds, when $p$ approaches one.
}

\blue{
Let us start by defining the hypothesis set that we want to investigate. Following \cite{CorMohRos10}, we consider the following hypothesis class  for $p\in[1,\infty]$:
\begin{equation*}
 \H_M^p :=  \left\{ h:\mathcal X\rightarrow \mathbb R ~ \bigg| ~ h(\x)= \sum_{m=1}^M \sqrt{\theta_m}\langle\w_m,\psi_m(\x)\rangle_{\mathcal H_m}, ~ \Vert\w\Vert_{\mathcal H}\leq 1, ~ \Vert\vtheta\Vert_p \leq 1 \right\}.
\end{equation*}
Solving our primal MKL problem \refprimalcounter corresponds to empirical risk minimization in the above hypothesis class.
We are thus interested in bounding the generalization error of the above class w.r.t.\ an i.i.d.\ sample $(\x_1,y_1),...,(\x_n,y_n)\in\mathcal X\times\{-1,1\}$ 
from an arbitrary distribution $P=P_X\times P_Y$.
In order to do so, we compute the \emph{Rademacher complexity},
$$\R(\H_M^p)  := \E \bigg[ \sup_{h \in \H_M^p} \frac{1}{n} \sum_{i=1}^n \sigma_i h(\x_i) \bigg], $$
where $\sigma_1, \ldots, \sigma_n$ are independent Rademacher variables (i.e. they obtain the values -1 or +1 with the same probability 0.5) and the $\E$ is the expectation operator that removes the dependency on all random variables, i.e. $\sigma_i$, $\x_i$, and $y_i$ ($i=1,...,n$).
If the  Rademacher complexity is known, there is a large body of results which can be used to bound the generalization error \cite[e.g.,][]{KolPan02,BarMen02}.
}

\blue{
We now show a simple $\ell_1$-to-$\ell_p$ conversion technique for the Rademacher complexity, which is the main result of this section:
\begin{theorem}[$\ell_1$-to-$\ell_p$ Conversion]\label{MainTheorem}
  For any sample of size $n$ and $p\in[1,\infty]$, the  Rademacher complexity of the hypothesis set 
  $\H_M^p$ can be bounded as follows:
  $$\R(\H_M^p)\leq \sqrt{M^{1/p^*}} \R(\H_M^1) ,$$
  where $p^*:=p/(p-1)$ is the conjugated exponent of $p$.
\end{theorem}
\begin{proof}
By H\"older's inequality \cite[e.g.,][]{SteeleBook}, we have
\begin{equation}\label{hoelder_trick}
  \forall \vtheta\in\mathbb R^M: ~\quad\Vert\vtheta\Vert_1 = \one^\top\vtheta \leq \Vert \one\Vert_{p^*} \Vert\vtheta\Vert_p = M^{1/p^*}\Vert\vtheta\Vert_p~. \hspace{0.5cm}
\end{equation}
Hence,
\begin{eqnarray*}
 \R(\H_M^p) & \stackrel{\rm Def.}{=} &  \E\bigg[ \sup_{\w: \Vert\w\Vert_{\mathcal H}\leq 1, ~ \vtheta:\Vert\vtheta\Vert_p \leq 1} \frac{1}{n} \sum_{i=1}^n \sigma_i \sum_{m=1}^M \sqrt{\theta_m}\langle\w_m,\psi_m(\x_i)\rangle_{\mathcal H_m} \bigg] \\
 & \stackrel{\eqref{hoelder_trick}}{\leq} & \E \bigg[ \sup_{\w:\Vert\w\Vert_{\mathcal H}\leq 1, ~ \vtheta:\Vert\vtheta\Vert_1 \leq M^{1/p^*}} \frac{1}{n} \sum_{i=1}^n \sigma_i \sum_{m=1}^M \sqrt{\theta_m}\langle\w_m,\psi_m(\x_i)\rangle_{\mathcal H_m} \bigg] \\
 & = & \E \bigg[ \sup_{\w:\Vert\w\Vert_{\mathcal H}\leq 1, ~ \vtheta:\Vert\vtheta\Vert_1 \leq 1} \frac{1}{n} \sum_{i=1}^n \sigma_i \sum_{m=1}^M \sqrt{\theta_m M^{1/p^*}}\langle\w_m,\psi_m(\x)\rangle_{\mathcal H_m}\bigg] \\
 & \stackrel{\rm Def.}{=} &\sqrt{M^{1/p^*}} \R(\H_M^1) .
\end{eqnarray*}
\end{proof}
\begin{remark}\label{theremark}
More generally we have that for any norm $\Vert\cdot\Vert_\star$ on $\mathbb R^M$, because all norms on $\mathbb R^M$ are equivalent \cite[e.g.,][]{Rud91}, there exists a $c_\star\in\mathbb R$ such that
$$\R(\H_M^p)\leq c_\star\R(\H_M^{\star}) .$$
This means the conversion technique extends to arbitrary norms: for any given norm ${\Vert\cdot\Vert_\star}$, we can convert any bound on $\R(\H_M^p)$ into a bound on the  Rademacher complexity $\R(\H_M^{\star})$ of hypothesis set induced by $\Vert\cdot\Vert_\star$.
\end{remark}
A nice thing about the above bound is that we can make use of any existing bound on the  Rademacher complexity of $H_M^1$ in order to obtain a generalization bound for $\H_M^p$.
This fact is illustrated in the following.
For example, the tightest result bounding $\R(\H_M^1)$ known so far is:
\begin{theorem}[\cite{CorMohRos10}]\label{thm:cortes}
    Let $M>1$ and assume that $k_m(\x,\x)\leq R^2$ for all $\x\in\mathcal X$ and $m=1,\ldots,M$. Then, for any sample of size $n$, the  Rademacher complexity of the hypothesis set 
    $\H_M^1$ can be bounded as follows (where $\c:=23/22$):
    $$\R(\H_M^1)\leq \sqrt{\frac{\c e\lceil\log M\rceil R^2}{n}} .$$
\end{theorem}
The above result directly leads to a $O(\sqrt{\log M})$ bound on the generalization error and thus substantially improves on a series of loose results given within the past years \cite[see][and references therein]{CorMohRos10}. We can use the above result (or any other similar result\footnote{The point here is that we could use any $\ell_1$-bound, for example, the bounds of \cite{KakShaTew10} and \cite{KloRueBar10} have the same favorable $O(\log M)$ rate; in particular, whenever a new $\ell_1$-bound is proven, we can plug it into our conversion technique to obtain a new bound.}) to obtain a bound for $\H_M^p$:
\begin{corollary*}[of the previous two theorems]
  Let $M>1$ and assume that $k_m(\x,\x)\leq R^2$ for all $\x\in\mathcal X$ and $m=1,\ldots,M$. Then, for any sample of size $n$, the  Rademacher complexity of the hypothesis set 
  $\H_M^1$ can be bounded as follows: 
  $$ \forall p \in [1,...,\infty]: \quad \R(\H_M^p)\leq \sqrt{\frac{\c e M^{1/p^*}\lceil\log M\rceil R^2}{n}} ,$$
  where $p^*:=p/(p-1)$ is the conjugated exponent of $p$ and $\c:=23/22$.
\end{corollary*}
It is instructive to compare the above bound, which we obtained by our $\ell_1$-to-$\ell_p$ conversion technique, with the one given in \cite{CorMohRos10}: that is $\R(\H_M^p)\leq \sqrt{\frac{\c e p^* M^{1/p^*} R^2}{n}}$ for any $p\in[1,...,\infty]$ such that $p^*$ is an integer. 
First, we observe that for $p=2$ the bounds' rates almost coincide: they only differ by a $\log M$-factor, which is  unsubstantial due to the presence of a polynomial term that domiates the asymptotics. Second, we observe that for small $p$ (close to one), the $p^*$-factor in the Cortes-bound leads to considerably high constants. When $p$ approaches one,
it even diverges to infinity. In contrast, our bound converges to  $\R(\H_M^p)\leq \sqrt{\frac{\c e \lceil\log M\rceil R^2}{n}}$ when $p$ approaches one, which is precisely the tight 1-norm bound of Thm.~\ref{thm:cortes}. Finally, it is also interesting to consider the case $p\geq 2$ (which is not covered by the \cite{CorMohRos10} bound): if we let $p\rightarrow\infty$, we obtain
$\R(\H_M^p)\leq \sqrt{\frac{\c e M\lceil\log M\rceil R^2}{n}}$. Beside the unsubstantial $\log M$-factor, our so obtained ${\cal O}\left(\sqrt{M}\ln(M)\right)$ bound matches the well-known  ${\cal O}\left(\sqrt{M}\right)$ lower bounds based on the VC-dimension \cite[e.g.,][Section 14]{DevGyoLug96}.
}

\blue{
We now make use of the above analysis of the Rademacher complexity  to bound the generalization error. There are many results in the literature that can be employed to this aim. Ours is based on 
Thm.~7 in \cite{BarMen02}:
\begin{corollary}\label{eq:loss_bound}
Let $M>1$ and $p\in[1,...,\infty]$. Assume that $k_m(\x,\x)\leq R^2$ for all $\x\in\mathcal X$ and $m=1,\ldots,M$. 
Assume the loss $\loss:\mathbb R\rightarrow [0,1]$ is Lipschitz with constant $L$ and $\loss(t)\geq 1$ for all $t\leq 0$.  Set $p^*:=p/(p-1)$ and  $\c:=23/22$.
Then, the following holds with probability larger than $1-\delta$ over samples of size $n$ for all classifiers $h \in \H_M^p$: 
\begin{align}
   R(h) \leq \widehat{R}(h) + 2L \sqrt{\frac{\c e M^{1/p^*}\lceil\log M\rceil R^2}{n}}  +  \sqrt{\frac{ \ln (2/\delta)}{2n}} ,
\end{align}
where $R(h)={\rm P}\big[yh(\x)\leq0\big]$ is the expected risk w.r.t. 0-1 loss and $\widehat{R}(h)=\frac{1}{n}\sum_{i=1}^n \loss( y_i h(x_i))$ is the empirical risk w.r.t. loss $V$.
\end{corollary}
The above theorem is formulated for general Lipschitz loss functions. Since the margin loss $V(t)=\min\big(1,[1-t/\margin]_+\big)$ is Lipschitz with constant $1/\gamma$ and upper bounding the 0-1 loss, it 
fulfills the preliminaries of the above corollary. Hence, we immediately obtain the following radius-margin bound \cite[see also][]{KolPan02}:
\begin{corollary}[$\ell_p$-norm MKL Radius-Margin Bound]
Fix the margin $\margin>0$. Let $M>1$ and $p\in[1,...,\infty]$. Assume that $k_m(\x,\x)\leq R^2$ for all $\x\in\mathcal X$ and $m=1,\ldots,M$. 
 Set $p^*:=p/(p-1)$ and  $\c:=23/22$.
Then, the following holds with probability larger than $1-\delta$ over samples of size $n$ for all classifiers $h \in \H_M^p$: 
\begin{align}
   R(h) \leq \widehat{R}(h) +  \frac{2R}{\margin}\sqrt{\frac{\c e M^{1/p^*}\lceil\log M\rceil }{n}}  +  \sqrt{\frac{ \ln(2/\delta)}{2n}} ,
\end{align}
where $R(h)={\rm P}\big[yh(\x)\leq0\big]$ is the expected risk w.r.t. 0-1 loss and $\widehat{R}(h)=\frac{1}{n}\sum_{i=1}^n \min\big(1,[1-y_i h(\x_i)/\margin]_+\big)$ the empirical risk w.r.t. margin loss.
\end{corollary}
}

\blue{
Finally, we would like to point out that, for reasons stated in Remark~\ref{theremark}, our $\ell_1$-to-$\ell_p$ conversion technique lets us easily extend the above bounds 
to norms different than $\ell_p$\hspace{0.8mm}. This includes, for example, block norms and sums of block norms as used in elastic-net regularization \cite[see][for such bounds]{KloRueBar10}, but also non-isotropic norms such as weighted $\ell_p$-norms.
}


\blue{
\subsection{Case-based Analysis of a Sparse and a Non-Sparse Scenario}
From the results given in the last section it seems that it is beneficial to use a sparsity-inducing $\ell_1$-norm penalty when learning with multiple kernels.
This however somewhat contradicts our empirical evaluation, which indicated that the optimal norm parameter $p$ depends on the true underlying sparsity of 
the problem. Indeed, as we show below, a refined theoretical analysis supports this intuitive claim. 
We show that if the underlying truth is uniformly non-sparse, then a priori there is no $p$-norm which is more promising than another one.
On the other hand, we illustrate that in a sparse scenario, the sparsity-inducing $\ell_1$-norm indeed can be beneficial. 
}

\blue{
We start by reparametrizing our hypothesis set based on block norms: by  Prop.~\ref{prop:block-norm} it holds that
$$  \H_M^p =  \left\{ h:\mathcal X\rightarrow \mathbb R ~ \bigg| ~ h(\x)= \sum_{m=1}^M \langle\w_m,\psi_m(\x)\rangle_{\mathcal H_m}, ~ \Vert\w\Vert_{2,q}\leq 1, ~ q:=2p/(p+1) \right\}, $$
where $||\w||_{2,q}:= \left(\sum_{m=1}^M ||\w_m||_{\mathcal H_m}^q\right)^{1/q}$ is the $\ell_{2,q}$-block norm. 
This means we can equivalently parametrize our hypothesis set in terms of block norms. 
Second, let us generalize the set by introducing an additional parameter $C$ as follows
$$  {\text{\small{$^C$}}}{\H_M^p} :=  \left\{ h:\mathcal X\rightarrow \mathbb R ~ \bigg| ~ h(\x)= \sum_{m=1}^M \langle\w_m,\psi_m(\x)\rangle_{\mathcal H_m}, ~ \Vert\w\Vert_{2,q}\leq C, ~ q:=2p/(p+1) \right\}. $$
Clearly,  ${\text{\small{$^{C}$}}}\H_M^p=\H_M^p$ for $C=1$, which explains why the  parametrization via $C$ is more general.
It is straight forward to verify  that $\R\left({\text{\small{$^C$}}}\H_M^p\right) = C\R\left(\H_M^p\right)$ for any $C$.
Hence, under the preliminaries of Corollary~\ref{eq:loss_bound}, we have
\begin{equation}\label{eq:bayes_boundEx1}
 R(h) \leq \widehat{R}(h) + 2L \sqrt{\frac{\c e M^{1/p^*}\lceil\log M\rceil R^2C^2}{n}}  +  \sqrt{\frac{ \ln(2/\delta)}{2n}}. 
\end{equation}
We will exploit the above bound in the following two illustrate examples.
\begin{figure}[h]
  \centering
  \subfigure{  \vspace{-1cm}\includegraphics[trim = 0mm -50mm 0mm 0mm,width=0.37\textwidth]{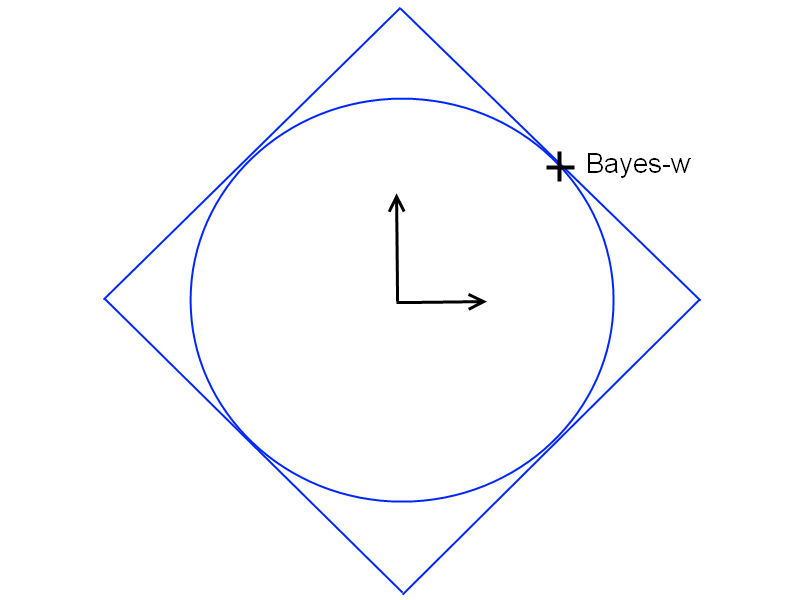}}
  \subfigure{ \includegraphics[width=0.37\textwidth]{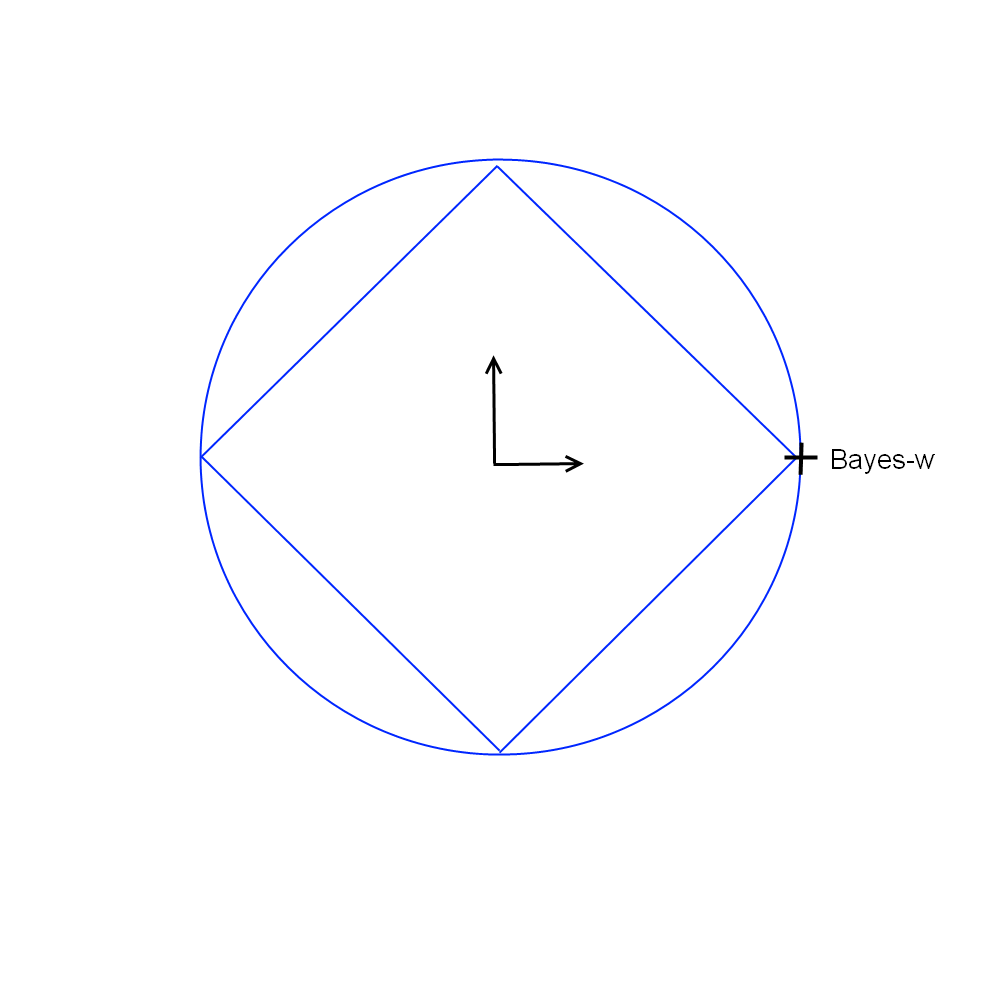}}
  \vspace{-1cm}
    \caption{ \small \label{fig:ex}
    Illustration of the two analyzed cases: a uniformly non-sparse (Example 1, left) and a sparse (Example 2, right) Scenario.
  }
\end{figure}
}

\blue{
\paragraph{Example 1.}
Let the input space be $\mathcal X=\mathbb R^M$, and the feature map be $\psi_m(\x)= x_m$ for all $m=1,\ldots,M$ and $\x=(x_1,...,x_M)\in\mathcal X$ (in other words, $\psi_m$ is a projection on the $m$th feature). Assume that the Bayes-optimal classifier is given by
$$ \w_{\rm Bayes} = (1,\ldots,1)^\top\in\mathbb R^M .$$
This means the best classifier possible is uniformly non-sparse (see Fig.~\ref{fig:ex}, left).
Clearly, it can be advantageous to work with a hypothesis set that is rich enough to contain the Bayes classifier, i.e.\ 
$(1,\ldots,1)^\top\in {\text{\small{$^C$}}}\H_M^p$.
In our example, this is the case if and only if $\Vert(1,\ldots,1)^\top\Vert_{2p/(p+1)}\leq C$, which itself is equivalent to $M^{(p+1)/2p}\leq C$.
The bound \eqref{eq:bayes_boundEx1} attains its minimal value under the latter constraint for $M^{(p+1)/2p}= C$. 
Resubstitution into the bound yields
$$ R(h) \leq \widehat{R}(h) + 2L \sqrt{\frac{\c e M^{2}\lceil\log M\rceil R^2}{n}}  +  \sqrt{\frac{ \ln(2/\delta)}{2n}}. $$
Interestingly, the obtained bound does not depend on the norm parameter $p$ at all! This means that in this particular (non-sparse) example
all $p$-norm MKL variants yield the same generalization bound. There is thus no theoretical evidence which norm to prefer a priori.
}

\blue{
\paragraph{Example 2.}
In this second example we  consider the same input space and kernels as before. 
But this time we assume a \emph{sparse} Bayes-optimal classifier (see Fig.~\ref{fig:ex}, right)
$$ \w_{\rm Bayes} = (1,0,\ldots,0)^\top\in\mathbb R^M .$$
As in the previous example, in order $\w_{\rm Bayes}$ to be in the hypothesis set, we have to require
$\Vert(1,0,\ldots,0)^\top\Vert_{2p/(p+1)}\leq C$. But this time this simply solves to $C\geq 1$, which is independent of the norm parameter $p$.
Thus, inserting $C=1$ in the bound \eqref{eq:bayes_boundEx1}, we obtain 
$$R(h) \leq \widehat{R}(h) + 2L \sqrt{\frac{\c e M^{1/p^*}\lceil\log M\rceil R^2}{n}}  +  \sqrt{\frac{ \ln(2/\delta)}{2n}},$$
which is precisely the bound of Corollary~\ref{eq:loss_bound}. It is minimized for $p=1$; thus, in this particular sparse example,
the bound is considerably smaller for sparse MKL---especially, if the number of kernels is high compared to the sample size. 
This is also intuitive: if the underlying truth is sparse, we expect a sparsity-inducing norm to match well the ground truth.
}

\blue{
\paragraph{}We conclude from the previous two examples that the optimal norm parameter $p$ depends on the underlying ground truth:
if it is sparse, then choosing a sparse regularization is beneficial; otherwise, any norm $p$ can perform 
well. I.e., without any domain knowledge there is no norm that a priori should be preferred. 
Remarkably, this still holds when we increase the number of kernels. 
This is somewhat contrary to anecdotal reports, which claim that sparsity-inducing norms are beneficial in high (kernel) dimensions.
This is because those analyses implicitly assume the ground truth to be sparse. 
The present paper, however, clearly shows that we might encounter a non-sparse ground truth in practical applications (see experimental section).
}

\section{Switching between Tikhonov and Ivanov Regularization}\label{APP-proofs}

In this appendix, we show a useful result that justifies switching from Tikhonov to Ivanov regularization and vice versa, if the bound on the regularizing constraint is tight. 
It is the key ingredient of the proof of Theorem \ref{th:reg-obj}.
We state the result for arbitrary convex functions, so that it can be applied beyond the multiple kernel learning framework of this paper.
\begin{proposition}\label{prop:pareto}
Let $D\subset\mathbb R^d$ be a convex set, let $f,g:D\rightarrow\mathbb R$ be convex functions. Consider the convex optimization
tasks
\begin{alignat}{3}
  \refstepcounter{equation}
     & ~ \min_{\x\in D} \quad && f(\x)+\sigma g(\x) \label{op:prop(a)} \tag{\theequation a}, \\
     &\min_{\x\in D:g(\x)\leq \tau} \quad \quad \quad && f(\x)  \label{op:prop(b)}\tag{\theequation b} .
\end{alignat}
Assume that the minima exist and that a constraint qualification holds in \eqref{op:prop(b)},
which gives rise to strong duality, e.g., that Slater's condition is
satisfied.  Furthermore assume that the constraint is active at the optimal point, i.e.
\begin{equation}\label{eq:constr_active}
  \inf_{\x\in D} ~  f(\x) ~ <  ~ \inf_{\x\in D:g(\x)\leq \tau} ~  f(\x)  .
\end{equation}
Then we have that for each $\sigma > 0$ there exists $\tau > 0$---and
vice versa---such that OP~\eqref{op:prop(a)} is equivalent to
OP~\eqref{op:prop(b)}, i.e., each optimal solution of one is an
optimal solution of the other, and vice versa.
\end{proposition}

\begin{proof}
\hspace{0cm}\\
(a).\quad
Let be $\sigma>0$ and $\x^*$ be the optimal of \eqref{op:prop(a)}. We have to show that there exists a 
$\tau> 0$ such that $\x^*$ is optimal in \eqref{op:prop(b)}.
We set $\tau=g(\x^*)$. Suppose $\x^*$ is not optimal in \eqref{op:prop(b)}, i.e.,
it exists $\tilde{\x}\in D: g(\tilde{\x})\leq\tau$ such that 
$f(\tilde{\x})< f(\x^*)$. Then we have 
$$ f(\tilde{\x})+\sigma g(\tilde{\x}) <  f(\x^*) + \sigma \tau, $$
which by $\tau=g(\x^*)$ translates to
$$ f(\tilde{\x})+\sigma g(\tilde{\x})<  f(\x^*) + \sigma g(\x^*) . $$
This contradics the optimality of $\x^*$ in \eqref{op:prop(a)},
and hence shows that $\x^*$ is optimal in \eqref{op:prop(b)}, which was to be shown. \\
(b).\quad Vice versa, let $\tau> 0$ be $\x^*$ optimal in \eqref{op:prop(b)}. The Lagrangian of \eqref{op:prop(b)} is given by 
$$\mathcal L(\sigma) = f(\x)+\sigma \left( g(\x)-\tau\right) , \quad \sigma\geq 0 .$$
By strong duality $\x^*$ is optimal in the saddle point problem
$$ \sigma^*:=\argmax_{\sigma\geq 0} ~ \min_{\x\in D} \quad f(\x)+\sigma \left(g(\x)-\tau\right) ,$$ 
and by the strong max-min property (cf. \citep{BoyVan04}, p.~238) we may exchange the order of maximization and minimization. 
Hence $\x^*$ is optimal in 
\begin{equation}\label{eq:sattlepoint}
 \min_{\x\in D} \quad f(\x)+\sigma^* \left( g(\x)-\tau\right) .
\end{equation}
Removing the constant term $-\sigma^*\tau$, and setting $\sigma=\sigma^*$, 
we have that $\x^*$ is optimal in \eqref{op:prop(a)}, which was to be shown.
Moreover by \eqref{eq:constr_active} we have that
$$\x^*\neq \argmin_{\x\in D} f(\x),$$
and hence we see from Eq.~\eqref{eq:sattlepoint} that $\sigma^*> 0$, which completes the proof of the proposition.
\end{proof}

\bibliography{tr}

\end{document}